\documentclass[letterpaper, 10 pt, conference]{ieeeconf}  

\IEEEoverridecommandlockouts                              

\usepackage{graphicx} 

\usepackage[hidelinks,bookmarks=true]{hyperref}

\usepackage{amsmath,amssymb,amsfonts}
\usepackage{bm}
\usepackage{xspace}
\usepackage{color}
\usepackage{xcolor}
\usepackage{graphicx,subcaption}
\usepackage{hyperref}
\usepackage{algorithm}
\usepackage{algorithmicx}
\usepackage[noend]{algpseudocode}
\algrenewcommand\algorithmicindent{0.75em}%
\usepackage{multirow}
\usepackage[table,xcdraw]{xcolor}
\usepackage{graphicx}
\usepackage{float}
\usepackage{caption}
\usepackage{subcaption}
\usepackage{comment}

\newboolean{PNGFig}
\setboolean{PNGFig}{false}

\ifthenelse{\boolean{PNGFig}}{
    \newcommand{\figType}[2]{#1}
}{
    \newcommand{\figType}[2]{#2}
}

\newcommand{\calO}{\ensuremath{\mathcal{O}}\xspace}

\newcommand{\calR}{\ensuremath{\mathcal{R}}\xspace}

\newcommand{\calW}{\ensuremath{\mathcal{W}}\xspace}
\newcommand{\calX}{\ensuremath{\mathcal{X}}\xspace}

\newcommand{\R}{\mathbb{R}}

\def\eps{\varepsilon}

\newcommand{\Cpp}{C\raise.08ex\hbox{\tt ++}\xspace}

\newtheorem{theorem}{Theorem}
\newtheorem{theorem*}{Theorem}

\newtheorem{defin}{Definition}
\newtheorem{prob}{Problem}

\newcommand\algname[1]{\textsf{#1}\xspace}
\newcommand\astar{\algname{A$^*$}}

\newcommand{\apex}{\algname{A$^*$pex}}

\newcommand{\ppastar}{\algname{PP-A$^*$}}

\newcommand{\cost}{\ensuremath{\mathbf{c}}\xspace}

\newcommand{\ignore}[1]{}

\newboolean{HIDENOTES}
\setboolean{HIDENOTES}{false}

\ifthenelse{\boolean{HIDENOTES}}{
    \newcommand{\OS}[1]{{}}
    \newcommand{\HP}[1]{{}}
    \newcommand{\EW}[1]{{}}
}{
    \newcommand{\OS}[1]{\textcolor{orange}{\textbf{OS:} #1}}
    \newcommand{\HP}[1]{{\textcolor{blue}{\textbf{HP:} #1}}}    
    \newcommand{\EW}[1]{{\textcolor{red}{\textbf{EW:} #1}}}
}

\newcommand\shape{\texttt{Shape}\xspace}

\usepackage{hyperref}

\usepackage[toc,page]{appendix}
\usepackage{wrapfig}

\usepackage[noend]{algpseudocode}
\usepackage{algorithm}
\algrenewcommand\algorithmicindent{0.7em}%

\usepackage{booktabs}
\usepackage[table]{xcolor}
\usepackage{adjustbox}

\begin{document}

\title{Generalizing Multi-Objective Search via \\
Objective-Aggregation Functions
}

\author{
Hadar Peer$^{1}$, 
Eyal Weiss$^{1}$, 
Ron Alterovitz$^{2}$, 
Oren Salzman$^{1}$%
\thanks{$^{1}$Technion, Israel Institute of Technology, Haifa, Israel.
}%
\thanks{$^{2}$University of North Carolina at Chapel Hill, NC, USA.
}%
}




\maketitle

\begin{abstract}
Multi-objective search (MOS) has become essential in robotics, as real-world robotic systems need to simultaneously balance multiple, often conflicting objectives. 
Recent works explore complex interactions between objectives, leading to problem formulations that do not allow the usage of out-of-the-box state-of-the-art MOS algorithms.
In this paper, we suggest a generalized problem formulation that optimizes solution objectives via aggregation functions of hidden (search) objectives. 
We show that our  formulation supports the application of standard MOS algorithms, necessitating only to properly extend several core operations to reflect the specific aggregation functions employed.
We demonstrate our approach in several diverse robotics planning problems, spanning motion-planning for navigation, manipulation and planning fr medical systems under obstacle uncertainty as well as inspection planning, and route planning with different road types.
We solve the problems using state-of-the-art MOS algorithms after properly extending their core operations, and provide empirical evidence that they outperform by orders of magnitude the vanilla versions of the algorithms applied to the same problems but without objective aggregation.
\end{abstract}

\section{Introduction}
\label{sec:intro}
Multi-objective search (MOS)~\cite{SalzmanF0ZCK23,tarapata2007selected,ulungu1991multi} has emerged as a fundamental tool in robotics, where systems must simultaneously satisfy diverse and often conflicting objectives. Examples include balancing path efficiency with safety, trading off energy consumption against task completion time, or optimizing accuracy while ensuring robustness to uncertainty. In such settings, the ability to reason over multiple objectives during search is indispensable.

Despite progress in MOS algorithms~\cite{hernandez2023simple,ren2025emoa}, existing formulations largely assume that objectives can be directly optimized in their raw form. However, recent research has demonstrated that this assumption is restrictive: In many robotics domains, objectives interact in ways that are not efficiently captured by independent optimization~\cite{FuKSA23,SlutskyYWF21}. Such interactions yield problem formulations that fall outside the direct applicability of standard MOS algorithms, limiting their effectiveness in realistic robotic planning scenarios.

\begin{figure*}[!t]
  \centering
  \subfloat[]{\includegraphics[height=0.125\textheight]{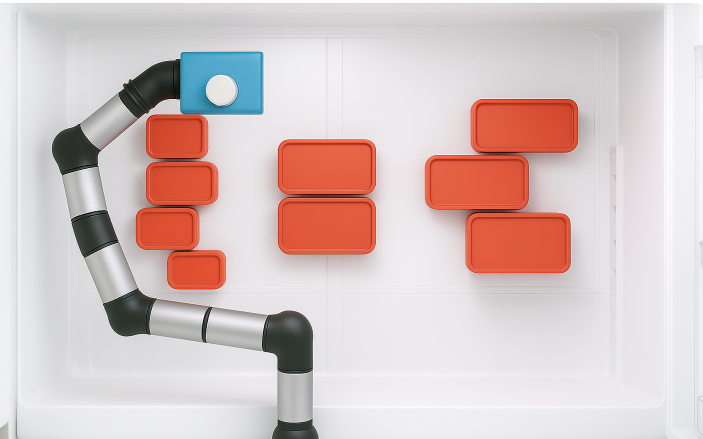}}\hfill
  \subfloat[]{\includegraphics[height=0.125\textheight]{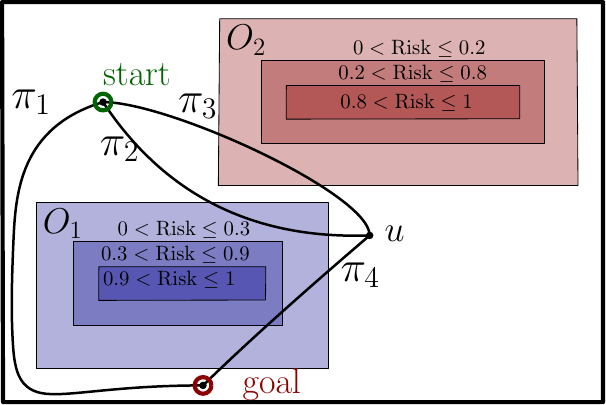}} \hfill
  \subfloat[]{\includegraphics[height=0.125\textheight]{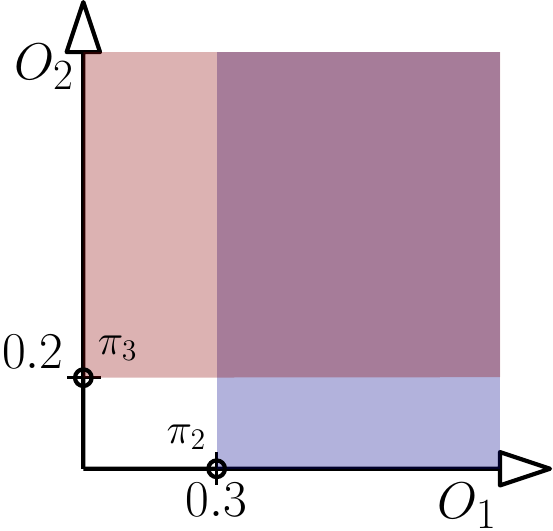}}\hfill
  \subfloat[]{\includegraphics[height=0.125\textheight]{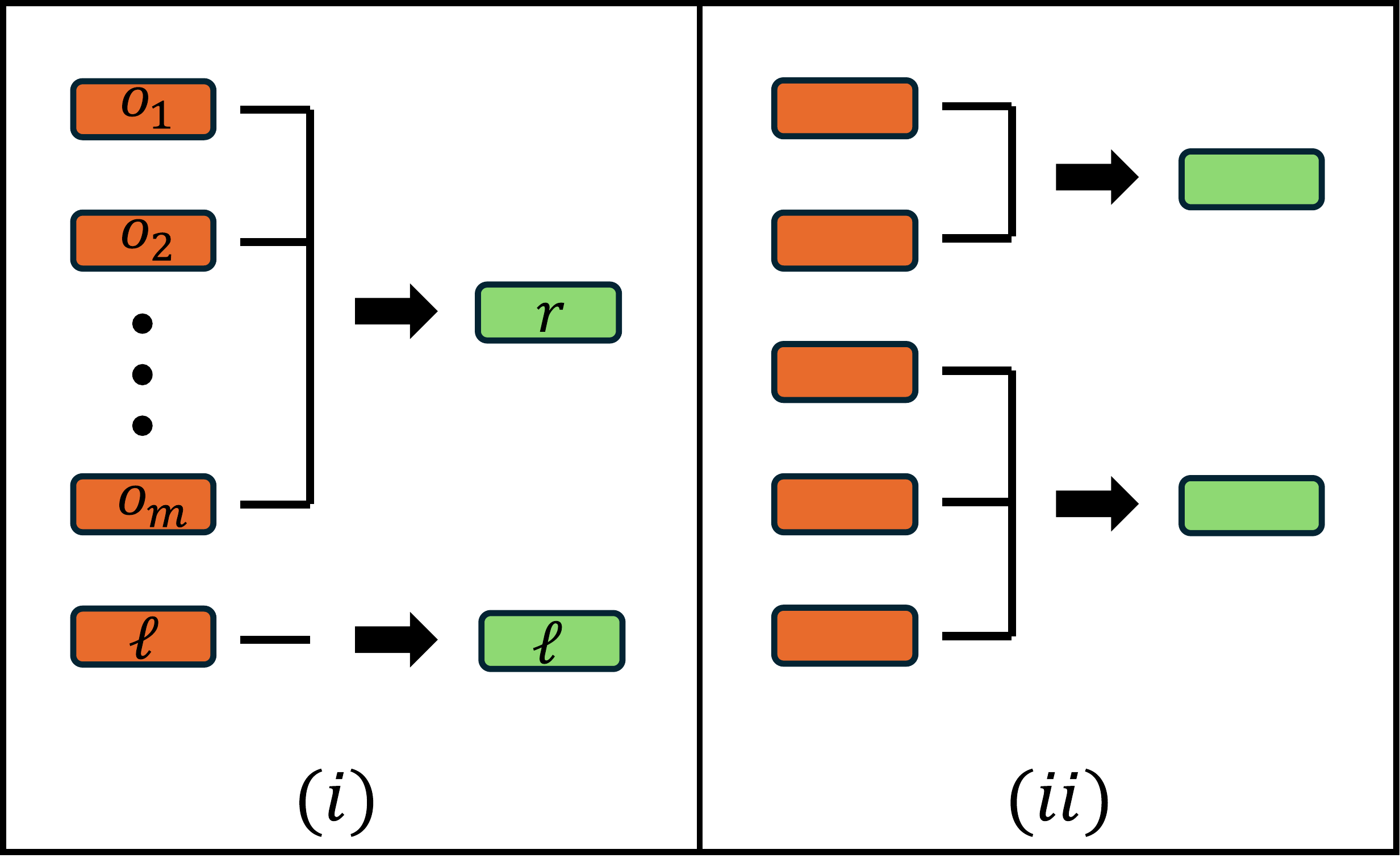}}

  \caption{\textbf{Motivating application.} 
  (a)~A robot manipulator tasked with grasping a milk carton in a cluttered refrigerator where the precise obstacle location is uncertain.
  (b)~Each obstacle in the configuration space is associated with a shadow, guaranteed to contain its true position  with a specified confidence level. Thus, the risk associated with a path corresponds to the ``deepest'' shadow it intersects. Here, path~$\pi_1$ is guaranteed to be collision free, 
  while paths $\pi_2, \pi_3, \pi_4$ have maximal collision probabilities of $0.3,0.2,0.9$, respectively.
  For $\pi_2 \cdot \pi_4$ and $\pi_3 \cdot \pi_4$ (here $\cdot$ denotes path concatenation), we have 
  maximal collision probabilities of
  $0.9$ and $1-(1-0.9)(1-0.2) = 0.92$, respectively. When multiple obstacles exists, the risk of each obstacle needs to be accounted for independently. Hence, from start to~$u$, path~$\pi_2$ has higher risk than~$\pi_3$, however from start to goal path~$\pi_2 \cdot \pi_4$ has lower risk than~$\pi_3 \cdot \pi_4$.
  (c)~Pareto-optimal frontier at~$u$.
  (d.i)~Per-obstacle risks $o_1,\ldots,o_m$ are aggregated into a single risk objective~$r$ (product) and paired with path length objective~$\ell$ to form the bi-objective $(r,\ell)$. (d.ii)~Example of a different aggregation.
  }
  \label{fig:motivation-risk-aggregation}
  \vspace{-6pt}
\end{figure*}

For instance, consider an application of motion-planning for a robot, in which it is desired to have \emph{short} paths and \emph{safe} paths, where safety is the complement of the risk, which corresponds to the collision probability of the robot with environment obstacles.
Recent work introduced the concept of \emph{shadows}~\cite{provably_safe}, which is a probabilistic geometric region that is guaranteed to contain the true position of an obstacle with a specified confidence level. The \emph{obstacle-level uncertainty} models obstacle uncertainty better  than  the standard and less realistic point-wise uncertainty. Here the overall safety is quantified through the product of non-collision probabilities with each of the obstacles in the environment.
Hence, to compute the safety of a path, one is required to reason about each obstacle \emph{independently}, where taking into account a shadow is done by monitoring the closest point along the path to the obstacle. This introduces interactions between obstacle-level collision risks and overall risk.
See Fig.~\ref{fig:motivation-risk-aggregation} for an illustration.

In MOS variants, as the one described above, the current best approach (see, e.g.,~\cite{FuKSA23}) is to cast the problem into standard MOS problems 
by introducing objectives that are required when considering partial paths but won't be part of the final solution (e.g., the risk of each obstacle in the example above will correspond to one objective). Once all solutions are computed, they can be post-processed (e.g., by aggregating risk among different obstacles in our running  example).
Unfortunately, the number of objectives required to model an underlying MOS problem may be large and the running time of MOS algorithms may grow exponentially with the number of objectives~\cite{bi-hard}. 
Thus, this approach is impractical for many settings, motivating more efficient problem representations.

We present a generalized MOS formulation that addresses the limitations of the approach described. Our key insight is to distinguish between \emph{hidden objectives}, which are incrementally computed during the search process, and \emph{solution objectives}, which are defined through aggregation functions over the hidden objectives. This abstraction enables the representation of a broad class of objective interactions while preserving compatibility with state-of-the-art MOS algorithms. Importantly, we show that only minimal algorithmic modifications are required—specifically, the extension of a small set of core operations to incorporate the aggregation functions.

We evaluate our new formulation in several diverse robotics planning problems, including various problems that require balancing path length and safety under obstacle uncertainty for scenarios inspired by 
navigation, manipulation and medical robotics. In these problems, the hidden objectives correspond to per-obstacle collision probabilities, which are aggregated to yield a goal objective of total collision probability. We instantiate the formulation in state-of-the-art MOS algorithms and provide empirical evidence that the extended variants outperform their standard counterparts when applied to these problems.

\textbf{Contribution.}
We present the following contributions: 
(i)~A generalized MOS problem formulation that separates hidden objectives from solution objectives via aggregation functions, enabling the modeling of complex objective interactions. 
(ii)~A demonstration that this formulation supports the application of state-of-the-art MOS algorithms with minimal changes to their core operations.
(iii)~Empirical evaluation spanning diverse problems depicting the efficacy of using this new formulation when compared to unmodified MOS algorithms.


\section{Related Work}
\label{sec:related_work}
In MOS, formally defined in Sec.~\ref{sec:mos}, we aim to compute a set of paths, called the \emph{Pareto-optimal frontier} (POF), that balance the trade-off between a given set of objectives. Namely, the set of paths where  no single path is strictly better than the others among all objectives. 
Exact state-of-the-art algorithms that compute the POF (see e.g.,~\cite{hernandez2023simple,mandow2010multiobjective,pulido2015dimensionality,ren2025emoa}) often build upon the celebrated~\astar algorithm~\cite{HNR68}.
Of particular interest to our work is 
\algname{NAMOA-dr}~\cite{pulido2015dimensionality} which can handle an  arbitrary number of objectives and is used in our evaluation.

In real-world settings we are often not interested in the entire POF---it's too large to present to decision makers and  computationally demanding to compute as the general problem is NP-hard~\cite{bi-hard}. 
In such settings, we seek  a bounded approximation of the POF, formally defined in Sec.~\ref{sec:mos}, which is both small and efficient to compute. 

As we outline in Sec.~\ref{sec:mos}, there is a simple approach to turning exact MOS algorithms to approximate ones. 
However, tailored, arguably more complex algorithms exists. They often  allow to approximate the POF much faster than the aforementioned simple adaption.
Notable examples include \ppastar~\cite {Goldin_Salzman_2021} and 
\apex~\cite{zhang2022pex}. Their efficiency stems from the observation that paths whose cost is very similar can be grouped in an efficient manner allowing to dramatically prune the POF.


Our work is motivated by recent endeavors that explore complex interactions between objectives in MOS.
This includes a number of papers on lexicographic and hierarchical search, that reflect preferences between objectives, in both deterministic (see e.g.,~\cite{shan2019lexicographic} for a single agent motion-planning setting and~\cite{zhao2024lexicographic} for a multi-agent setting), and stochastic planning problem setups (see e.g.,~\cite{miura2022heuristic} for heuristic search in stochastic shortest-path and~\cite{wray2015multi} for reward preferences in POMDPs).
Other notable directions involve different forms of aggregation, such as scalarizing multi-objective planning using weighted maximization~\cite{wilde2024scalarizing}, non-additive shortest-path problems~\cite{durand2024generalized} and various reformulations of chance-constrained planning (e.g., scenario-based motion-planning ~\cite{de2025scenario}).
A common theme for all these works is the need to \emph{tailor} algorithmic techniques to fit each new setting. This is addressed by our novel framework, which unifies many problem settings and allows to tackle all of them with the same set of tools.

\section{Algorithmic Background} 
\label{sec:mos}
In this section we introduce the notation required to formally define the MOS problem and describe a general algorithmic framework that can be used to solve the problem. For additional details, see e.g.~\cite{SalzmanF0ZCK23,SkylerSAFSC0K0U24}.
\subsection{Notation}
\label{subsec:mos_notation}

Boldface font indicates vectors, lower-case and upper-case symbols indicate elements and sets, respectively. The notation~$p_i$ will be used to  denote the $i$'th component of~$\mathbf{p}$. 
The addition of two $d$-dimensional vectors $\mathbf{p}$ and~$\mathbf{q}$ and the multiplication
of a real-valued scalar $k$ and a $d$-dimensional vector $\mathbf{p}$ are defined in the natural
way, namely as $\mathbf{p}+\mathbf{q}=(p_1+q_1,\ldots, p_d+q_d)$ and
$k\mathbf{p}=(kp_1,\ldots,kp_d)$, respectively. 

Let $\mathbf{p}$ and $\mathbf{q}$ be $d$-dimensional vectors.
For a minimization problem, we say that~$\mathbf{p}$ 
\emph{dominates}~$\mathbf{q}$ and denote this as $\mathbf{p} \preceq \mathbf{q}$ if $\forall i, p_i \leq q_i$.
%
%
%
We say that $\mathbf{p}$ is \emph{lexicographically  smaller} than~$\mathbf{q}$ and denote this as $\mathbf{p} \prec_{\rm{lex}} \mathbf{q}$ if~${p}_k < {q}_k$ for the first index~$k$ s.t. ${p}_k \neq {q}_k$  (or if $\forall i, p_i \leq q_i$).
Finally, 
let $\mathbf{p}$ and $\mathbf{q}$ be  two $d$-dimensional vectors
and let~$\boldsymbol{\varepsilon}$ be another $d$-dimensional vector such that $\forall i~\eps_i \geq 0$.
We say that $\mathbf{p}$ \emph{approximately dominates}~$\mathbf{q}$ with an \emph{approximation factor} $\boldsymbol{\eps}$ and denote this as $\mathbf{p} \preceq_{\boldsymbol{\eps}} \mathbf{q}$ if $\forall i, p_i \leq (1+\eps_i) \cdot q_i$.

\subsection{Setting}
\label{subsec:mos_setting}
A MOS \emph{graph} $G$ is a tuple $(V,E,\cost)$, where~$V$ is the finite
set of \emph{vertices}, $E\subseteq V\times V$ is the finite set of \emph{edges}, and
$\cost:E\rightarrow \mathbb{R}^d_{\geq 0}$ is a
\emph{cost function} that associates a $d$-dimensional vector of non-negative real costs with
each edge. 
A \emph{path} $\pi$ from $v_1$ to $v_n$ is a sequence of vertices $v_1, v_2, \ldots, v_n$
such that $(v_i,v_{i+1})\in E$ for all $i\in\{1,\ldots,n-1\}$. 

We define the \emph{cost} of a path $\pi = v_1,\ldots,v_n$ as $\cost(\pi)=\sum_{i=1}^{n-1} \cost(s_i,s_{i+1})$.
Given paths~$\pi$ and $\pi'$, we extend all the above definitions to paths. 
E.g., we say that~$\pi$  \emph{dominates} $\pi'$ and denote this as $\pi \preceq \pi'$ if 
 $\cost(\pi) \preceq \cost(\pi')$. 

Let~$G=(V,E,\cost)$ be a MOS graph and~$v_{\rm s},v_{\rm g} \in V$  start and goal vertices.
A path from $v_{\rm s}$ to $v_{\rm g}$ is called a \emph{solution}. 
The POF $\Pi^*$ is the set of all solutions s.t. 
for every solution $\pi \in \Pi^*$, there does not exists a solution~$\pi'$  s.t.~$\pi' \preceq \pi$.

\subsection{Algorithms}
\label{subsec:mos_algorithms}
\begin{algorithm}[t]
\small
\caption{\algname{MOS-A$^*$}}\label{alg:mos}

\textbf{Input:} 
    $G = (V,E,\mathbf{c})$; 
    $v_{\rm s}, v_{\rm g}$; 
    $\mathbf{h}$
    
\begin{algorithmic}[1]
    \State $n_{\text{root}} \leftarrow$ new node at $v_{\rm s}$ with 
        $\mathbf{g}(n_{\text{root}} )=\mathbf{0}$, 
        $\mathbf{f}(n_{\text{root}} )=\mathbf{h}(v_{\rm s})$
        \label{alg1_line:init_root}
    \State \textsc{Open}.insert($n_{\text{root}}$);~~
           \textsc{Sols}$\leftarrow \emptyset$;~~
           \textsc{Closed}$\leftarrow \emptyset$
           \label{alg1_line:init_lists}

    \While{\textsc{Open} $\neq \emptyset$} \label{alg1_line:while_loop}
        \State $n \leftarrow $ \texttt{get\_best\_node}(\textsc{Open}) \label{alg1_line:update_n}
        \If{\texttt{is\_dominated\_sol}
            (n,\textsc{Sols}) 
            or  {\color{white}. \hspace{28 mm}}
                {\color{white}.\hspace{3 mm}}
            \texttt{is\_dominated\_path}
            (n, \textsc{Closed}) 
            }  
            \label{alg1_line:if_dominated}
            \State \textbf{continue}
            \label{alg1_line:continue_dominated}
        \EndIf

        \If{\texttt{is\_solution}($n$)} \label{alg1_line:if_solution}
            \State \textsc{Sols}.insert($n$) \label{alg1_line:insert_sol}
            \State \textbf{continue} \label{alg1_line:continue_solution}
        \EndIf

        \For{\textbf{each} $v'$ s.t. $(v(n), v') \in E$} \label{alg1_line:for_loop}
            \State $n' \leftarrow$ new node at $v'$ \label{alg1_line:update_child}
            \State $\mathbf{g}(n')\leftarrow$\texttt{compute\_cost}($\mathbf{g}(n),
            \mathbf{c} \left(v(n), v(n') \right)$) \label{alg1_line:update_g_value}
            \State $\mathbf{f}(n')\leftarrow \mathbf{g}(n') + \mathbf{h}(v')$ 
            \label{alg1_line:update_f_value}
            \State \textsc{Open}.insert($n'$) \label{alg1_line:insert_child}
        \EndFor
        \State \textsc{Closed}.insert($n$) \label{alg1_line:close_parent}
    \EndWhile

    \State \textbf{return} \textsc{Sols} \label{alg1_line:return}
\end{algorithmic}
\end{algorithm}


\begin{algorithm}[t]
\footnotesize
\caption{\algname{MOS-A$^*$} functions (MOS)}\label{alg:func}

\textbf{function:} 
        \texttt{get\_best\_node}
            (\textsc{Open})
        
\begin{algorithmic}[1]

    \State \textbf{return}  $\arg \min_{\prec_{\rm lex}}$ 
                            $\{$ 
                            ${\color{blue} \mathbf{f}(n)}~\vert~ n\in$  {\textsc{Open}}
                            $\}$
            \hfill        /*
            {\color{magenta} 
                $\mathcal{F}_{\rm agg}(\mathbf{f}(n))$
            }

\end{algorithmic}

\textbf{function:} 
        \texttt{is\_dominated\_sol}
            ($n'$, \textsc{Sols})
        
\begin{algorithmic}[1]
 
    \State 
        \textbf{if} 
        $\exists n \in $~\textsc{Sols} s.t. 
        ${\color{blue}\mathbf{f}(n) \preceq \mathbf{f}(n')}$
        \textbf{then}
        \hfill        /*
        {\color{magenta} 
            $\mathcal{F}_{\rm agg}(\mathbf{f}(n))
            \preceq
            \mathcal{F}_{\rm agg}(\mathbf{f}(n'))$
        }
    \State 
        {\color{white}.\hspace{0.9 mm}}
        \textbf{return} true
    
\end{algorithmic}

\textbf{function:} 
        \texttt{is\_dominated\_path}
            ($n'$, \textsc{Closed})
        
\begin{algorithmic}[1]
    \If{$\exists n \in $~\textsc{Closed} s.t. 
        $v(n) = v(n')$~and~$\mathbf{g}(n) \preceq \mathbf{g}(n')$}
        \State \textbf{return} true
    \EndIf
\end{algorithmic}

\textbf{function:} 
    \texttt{compute\_cost} 
        ($  \mathbf{g_{\rm parent}}, 
            \mathbf{c}
         $) 
\begin{algorithmic}[1]
    \State \textbf{return} ${\color{blue}\mathbf{g_{\rm parent}} + \mathbf{c}}$
    \hfill /*
    {\color{magenta} $\mathcal{F}_{\rm ext}(\mathbf{g_{\rm parent}},  \mathbf{c})$}
\end{algorithmic}

\end{algorithm}

Recall (Sec.~\ref{sec:related_work}) that state-of-the-art MOS algorithms are often \astar-based, thus we  restrict our discussion to this family of algorithms (and assume basic familiarity with \astar).

\subsubsection{Exact MOS}
Following Skyler et al.~\cite{SkylerSAFSC0K0U24} we describe in Alg.~\ref{alg:mos} a simplified \astar-based MOS algorithm.
As we shall see shortly, our new formulation will be able to utilize the same algorithm by surgically replacing key operations. These are highlighted in colored text. For now, we instruct the reader to ignore any text in {\color{magenta}  magenta}.

We now briefly explain Alg.~\ref{alg:mos} which runs a best-first search over nodes representing solutions, and utilizes dominance checks to prune dominated solutions that are not part of the POF.
In Line~\ref{alg1_line:init_root} the root node and its $g,f$-values are initialized according to the start vertex~$v_s$.
In Line~\ref{alg1_line:init_lists} the data structures \textsc{Open}, \textsc{Closed} and \textsc{Sols} are initialized to include the root node and empty sets respectively.
Line~\ref{alg1_line:while_loop} loops over the nodes in \textsc{Open} until exhaustion, in which \textsc{Sols}, the POF is returned at Line~\ref{alg1_line:return}.
Line~\ref{alg1_line:update_n} updates the current node $n$ by removing the best node in \textsc{Open} according to the function \texttt{get\_best\_node}, which orders the nodes by lexicographic order of $f$ values.
It is worth noting that utilizing a lexicographic order allows to effectively perform a dimensionality reduction that enhances the node comparison operation (see~\cite{pulido2015dimensionality}) and that efficient variants exist (see the overview~\cite{SalzmanF0ZCK23}).
Line~\ref{alg1_line:if_dominated} checks if the node $n$ is dominated by either a solution which was already found via the function \texttt{is\_dominated\_sol} or by a better path to the node $n$ via the function \texttt{is\_dominated\_path}.
If the node $n$ is dominated then at Line~\ref{alg1_line:continue_dominated} the node is discarded and the search continues.
Line~\ref{alg1_line:if_solution} checks if $n$ is a solution. 
If true, then it is added to  \textsc{Sols} and the search continues without expanding the node at Lines~\ref{alg1_line:insert_sol}-\ref{alg1_line:continue_solution}.
Line~\ref{alg1_line:for_loop} loops over the children (neighbors) of $n$ in the graph $G$.
After a new child is generated at Line~\ref{alg1_line:update_child}, its $g$-value is computed at Line~\ref{alg1_line:update_g_value} using the function \texttt{compute\_cost}.
Then, its $f$-value is computed at Line~\ref{alg1_line:update_f_value}, and inserted into \textsc{Open} at Line~\ref{alg1_line:insert_child}.
After $n$ is expanded, it is inserted into \textsc{Closed} at Line~\ref{alg1_line:close_parent}.

\subsubsection{Approximate MOS}
Recall (Sec.~\ref{sec:related_work}), that in many real-world settings, we are often interested in  a bounded approximation of $\Pi^*$ which is both small and efficient to compute.
To this end, given an approximation factor~$\boldsymbol{\eps}$, an \emph{$\boldsymbol{\eps}$-approximate POF} $\Pi^*_{\boldsymbol{\eps}}$ is a set of solutions such that every path in~$\Pi^*$ is $\boldsymbol{\eps}$-dominated by a path in~$\Pi^*_{\boldsymbol{\eps}}$. 

A general approach to compute~$\Pi^*_{\boldsymbol{\eps}}$ using an exact MOS algorithm (see e.g.,~\cite{GS21})
is to replace the domination test $\preceq$ used in solution domination (i.e., Line~1 in \texttt{is\_dominated\_sol}) with approximate domination test~$\preceq_{\boldsymbol{\varepsilon}}$.

\section{Generalizing MOS via Objective Aggregation} 
\label{sec:gen_mos}
As we will see in Sec.~\ref{sec:applications}, many applications do not adhere to the MOS setting described in Sec.~\ref{sec:mos}. Roughly speaking, this is because 
(i)~path cost is not necessarily additive and, more importantly,
(ii)~to compute a final solution, we may need to track multiple objectives for indeterminate nodes.
We start by introducing two new functions and use them to generalize the MOS problem (Sec.~\ref{sec:mos}). We then continue to describe a straw man algorithm to solve this problem and conclude by introducing a simple-yet-effective adaption to Alg.~\ref{alg:mos}.

\subsection{Setting}
\label{subsec:gen_mos_setting}
Let $G=(V,E,\mathbf{c})$ with $\cost:E\rightarrow \mathbb{R}^d_{\geq 0}$
be a MOS graph.
Consider monotonically non-decreasing functions 
$\mathcal{F}_{\rm ext} : \mathbb{R}^m \times \mathbb{R}^d \rightarrow \mathbb{R}^m$
and 
$\mathcal{F}_{\rm agg} : \mathbb{R}^m \rightarrow \mathbb{R}^k$,
denoted as 
a \emph{path-cost extension function} 
and 
an \emph{objective aggregation function}, respectively.
We call~$m$ and $k$ the 
\emph{hidden objective dimensionality}
and 
\emph{solution objective dimensionality}, respectively.

\paragraph{Path cost}
We generalize the notion of path cost as follows:
Given path $\pi = v_1,\ldots,v_n$, let $\pi^i = v_1,\ldots,v_i$ be the partial path composed of the first $i$ vertices of $\pi$ for some $i\in\{1, \ldots n \}$.
We define the \emph{cost w.r.t. a path-cost extension function} 
recursively as follows 
$\cost_{\mathcal{F}_{\rm ext}}(v_1):= \boldsymbol{0}$ (vector of zeros with dimension $m$) and for $i>1$ we define
$\cost_{\mathcal{F}_{\rm ext}}(\pi^i):= \mathcal{F}_{\rm ext} (\cost_{\mathcal{F}_{\rm ext}}(\pi^{i-1}),\cost(v_{i-1},v_i))$.

\vspace{1mm}
\noindent
\textbf{Note.}
When $\mathcal{F}_{\rm ext}:=(\mathbf{g}, \mathbf{g}')= \mathbf{g} + \mathbf{g}'$ with $m=d$, this is the standard definition of additive path cost. We call such a function a \emph{trivial} path-cost extension function.

\paragraph{Solution cost}
We generalize the notion of the POF as follows:
Given a solution $\pi$, we define \emph{solution cost w.r.t. an objective-aggregation function}
as $\cost_{\mathcal{F}_{\rm agg}}(\pi):=\mathcal{F}_{\rm agg}(\cost_{\mathcal{F}_{\rm ext}}(\pi))$.
Similarly, the 
\emph{POF  w.r.t. an objective-aggregation function}~%
$\Pi^*_{\mathcal{F}_{\rm agg}}$ is the set of all solutions such that ~$\forall \pi \in \Pi^*_{\mathcal{F}_{\rm agg}}$, there does not exists a solution~$\pi'$  
s.t.~$\cost_{\mathcal{F}_{\rm agg}}(\pi') \prec \cost_{\mathcal{F}_{\rm agg}}(\pi)$.

\vspace{1mm}
\noindent
\textbf{Note.}
When $\mathcal{F}_{\rm agg}$ is the identity function (and thus $m=k$), this is the standard definition of POF. We call such a function a \emph{trivial} objective-aggregation function.

\paragraph{Examples}
We describe below examples of non-trivial path extension and objective-aggregation functions. 
As we will see in Sec.~\ref{sec:applications}, these model natural robotic applications.
\begin{align}
\label{eq:ext1}
\mathcal{F}_{\rm ext}(\mathbf{c}, \mathbf{c}') = \big(
    &\max(c_1,c_1'), \ldots, \max(c_{m-1} , c_{m-1}'),\notag \\
    &c_m + c_m' \big),
\end{align}
\begin{equation}
\label{eq:agg1}
    \mathcal{F}_{\rm agg} (\mathbf{c}) = \left(1 - \Pi_{i=1}^{m-1}(1-c_i), c_m \right),
\end{equation}
\begin{equation}
\label{eq:agg2}
    \mathcal{F}_{\rm agg} (\mathbf{c}) = 
            \left( (m-1) - \sum_{i=1}^{m-1} c_i, c_m \right).
\end{equation}

Note that in Eq.~\eqref{eq:ext1} it holds that $m=d$, and in Equations~\eqref{eq:agg1}-\eqref{eq:agg2} we have $k=2$.
{These aggregations are captured by Fig.~\ref{fig:motivation-risk-aggregation}d.i. 
Fig.~\ref{fig:motivation-risk-aggregation}d.ii presents a more-general aggregation which,  for ease of presentation, are not included here.}

\subsection{Algorithms}
\label{subsec:gen_mos_algorithms}
Given a MOS problem with a non-trivial path extension function but a trivial objective-aggregation function, we can run almost any MOS algorithm. The only change that is required is updated path cost computation. Returning to our representative MOS algorithm (Alg.~\ref{alg:mos}), this corresponds to updating \texttt{compute\_cost} to use ${\mathcal{F}_{\rm ext}}$. This corresponds to replacing the {\color{blue} blue} text in the function with the {\color{magenta} magenta} text.

When the objective-aggregation function is non-trivial and we wish to compute $\Pi^*_{\mathcal{F}_{\rm agg}}$, one requires slightly more care.
A straw man approach to compute $\Pi^*_{\mathcal{F}_{\rm agg}}$ would be to run any  MOS algorithm and compute $\Pi^*$ which corresponds to the POF over the hidden objectives. Then, for every solution $\pi \in \Pi^*$ compute its solution cost $\cost_{\mathcal{F}_{\rm agg}}(\pi)$. Subsequently, we compute the undominated solutions set among these costs.

As the hidden objective dimensionality $m$ is often much larger than the solution objective dimensionality $k$, the size of the POF over the hidden objectives may be much larger than the size of $\Pi^*_{\mathcal{F}_{\rm agg}}$ rendering the straw man approach highly inefficient.
To this end, we highlight which changes need to be made to our representative MOS algorithm (Alg.~\ref{alg:mos}).

The first change required corresponds to  the ordering function of the priority queue \textsc{Open}.
To employ standard MOS tools such as dimensionality reduction~\cite{pulido2015dimensionality}, \textsc{Open} needs to be ordered lexicographically according to $\mathcal{F}_{\rm agg}(\mathbf{f}(n))$ and not lexicographically according to $\mathbf{f}(n)$. Returning to Alg.~\ref{alg:mos}, this corresponds to replacing the {\color{blue} blue} text with the {\color{magenta} magenta} text in  \texttt{get\_best\_node}.
This  means that the heuristic function $\boldsymbol{h}$ should also be computed according to the aggregate cost~$\mathcal{F}_{\rm agg}$.

The second change is concerned with domination checks. 
Since solution cost can't be computed for a node that is not a solution, path domination is performed w.r.t. the hidden objectives (i.e., no changes to function \texttt{is\_dominated\_path}).
Solution domination however, can be updated. 
Here, we simply need to check domination w.r.t.  $\mathcal{F}_{\rm agg}(\mathbf{f}(n))$. Again,  this corresponds to replacing the {\color{blue} blue} text with the {\color{magenta} magenta} text in function \texttt{is\_dominated\_solution}.

To differentiate between the approaches, we denote by $(k,m)$-$\algname{ALG}_\varepsilon$ an MOS algorithm \algname{ALG} together with approximate dominance test $\varepsilon$ using $k$ and $m$ hidden and solution objective dimensionality, respectively. 
The straw man approach is when $k=m$, whereas our proposed objective aggregation corresponds to~$k<m$.

\begin{theorem}
    Let \algname{ALG} be some MOS algorithm and consider a MOS problem with $k$ and $m$ solution and hidden objectives for some~$k<m$, respectively, and with 
    monotonically non-decreasing path extension and objective-aggregation functions $\mathcal{F}_{\rm ext}$ and $\mathcal{F}_{\rm agg}$.
    If $(m,m)$-\algname{ALG} returns a POF to the MOS problem
    then $(k,m)$-\algname{ALG} (using the {\color{magenta} magenta} text instead of the {\color{blue} blue} text in Alg.~\ref{alg:mos}) also returns a POF to the MOS problem.
\end{theorem}


\noindent
Proof omitted due to lack of space~
\footnote{https://tinyurl.com/mytheoremproof}.
However, it is interesting to note that the monotonicity requirement on $\mathcal{F}_{\rm ext}$ can be relaxed, as long as monotonicity is preserved for the solution objectives, defined by $\cost_{\mathcal{F}_{\rm agg}}(\pi):=\mathcal{F}_{\rm agg}(\cost_{\mathcal{F}_{\rm ext}}(\pi))$. We will present such a use case in Sec.~\ref{subsec:route}.





\section{Representative Applications} 
\label{sec:applications}
\subsection{Inspection Planning}
\label{subsec:ip}

In robot inspection planning~\cite{AlpertSKS25,FuKSA23} we are given a robot $\mathcal{R}$ and a set of $q$ points of interests (POI) $\mathcal{I} = \{ p_1,\, \ldots p_q \}$. $\mathcal{R}$ is equipped with a sensor $S$ and we are tasked with planning the motion of~$\mathcal{R}$ such that the sensor $S$ covers as many POIs as possible along its path and minimizes the length of the path. 
As is common in motion-planning (see, e.g.,~\cite{OrtheyCK24}), some works first sample a roadmap $G = (V,E)$ (a graph embedded in $\mathcal{R}$'s configuration space) and solve this discretized problem which is termed the \emph{graph inspection planning} problem.

Interestingly, Fu et al.~\cite{FuKSA23} suggested an algorithm that follows exactly our formulation (though stated in terms that are specific to graph inspection planning) and which inspired the \algname{A$^*$pex} algorithm.
Specifically, each POI $p_i$ induces a hidden objective $o_i$ that is set to one if it was sensed somewhere along a path and zero otherwise. 
Thus, we have here $m = q+1$ hidden objectives (one for each POI and one for path length). 
When extending path~$\pi$ by edge $e$,
the path extension function 
(i)~sets hidden objective $o_i$ for $i\in \{1, \ldots, q\}$ to one if it was sensed along~$\pi$ or along $e$
and
(ii)~accumulates path length. This is captured exactly by Eq.~\eqref{eq:ext1}.
The final objective is to minimize both the number of POIs that were not seen as well as to minimize the length. 
This is captured exactly by Eq.~\eqref{eq:agg2}. 
Notice that $c_i, i\in\{1,...,m-1\}$ in both equations are binary variables, s.t. $c_i\in\{0,1\}$.
For additional details, see~\cite{FuKSA23}. 

\subsection{Planning under Obstacle Uncertainty (OU)}
\label{subsec:shadows}
In the problem of planning under obstacle uncertainty (OU), also  discussed  in Sec.~\ref{sec:intro}, we are given a robot $\calR$ operating in workspace $\calW$ consisting of $q$ obstacles \(\mathcal{O} = \{ O_1, \ldots, O_q \}\). 
The exact location of each obstacle is associated with a probabilistic distribution termed \emph{shadow}~\cite{provably_safe} (to be formalized shortly) and we are tasked with planning the motion of $\calR$ while minimizing the risk of its path---defined in terms of collision probability with any obstacle in the environment---and its length. 

We note that similar problems which build upon shadow formulation such as
minimizing collision risk~\cite{axelrod2022efficient} or 
maintaining risk below  a specified threshold~\cite{Trajectory_optimization} 
were recently introduced.
As modeling obstacle uncertainty using shadows is a relatively new concept, for completeness,  we formally describe the problem.

A \emph{configuration} $x$ is a $t$-dimensional vector that uniquely identifies the position and orientation of $\calR$ in $\calW$.
The set~$\calX$ of all robot configurations is called the \emph{configuration space}. 
Let $\shape: \calX \rightarrow 2^\calW$ be a function that maps each configuration $x$ to the volume occupied by $\calR$ when placed at~$x$.

In contrast to standard motion-planning problems, where obstacles are known with certainty, we consider a setting in which obstacles may have uncertain shape, size, or location, modeled using a known distribution: \emph{shadows}.
\begin{defin}[Shadows~\cite{provably_safe}]
    For each obstacle $O_i \in \mathcal{O}$, we define a sequence of volumes called \emph{shadows},~$\mathcal{S}(O_i) = \{ S^i_0, S^i_1, \dots \}$,
    where $\forall i,j : S^i_j \subseteq \mathcal{W}, S^i_j \subseteq S^i_{j+1}$.
\end{defin}

Ideally, we would like to utilize the exact collision probability of~$\calR$ at~$x$ with an obstacle $O_i$. 
However, we do not assume that we do not have access to it.
Instead, let $r_{\text{max}}: \mathcal{S} \rightarrow [0,1]$ associate each shadow with an upper bound on the collision probability with its underlying obstacle.
Namely, for every configuration $x$, we have that,
$$r_{\text{max}}(S^i_j) \geq \text{Prob.}\big[\, \shape(x) \cap O_i \neq \emptyset~|~\shape(x) \cap S^i_j \neq \emptyset \,\big].$$
Then, we define the \emph{shadow's risk} as an upper bound on the collision probability
of~$\calR$ to intersect shadow $S^i_j$ of obstacle~$O_i$ when at configuration~$x$ to be $\texttt{risk}(x,S^i_j) = r_{\text{max}}(S^i_j)$. 
Here, we assume that risk monotonically increases for smaller shadows.
Namely, a smaller shadow has a \emph{higher} risk upper bound. Namely,~$\forall i,j, j'~s.t.~j<j'$ it holds that~$r_{\text{max}}(S^i_j)>r_{\text{max}}(S^i_{j'})$.



Note that the robot may simultaneously intersect several shadows of the same obstacle.
Thus, it is useful to define the notion of an \emph{effective shadow} of obstacle~$O_i$ which is the smallest shadow that intersects the robot.
Namely, the effective shadow of  $\calR$ at configuration $x$ for obstacle $O_i$ is defined as~$S_{\text{eff}}^i(x) :=
    S^i_{\, \arg\min_{j} \{\, j \mid \shape(x) \cap S^i_j \neq \emptyset \,\} }$.
Then, an \emph{obstacle's risk}, denoted by~$\texttt{risk}(x, O_i)$ is the risk of the effective shadow at~$x$. I.e.,~$\texttt{risk}(x, O_i) := \texttt{risk}(x,S_{\text{eff}}^i(x))$. 


A robot's \emph{path} $\pi: [0,1] \rightarrow \calX$ is a continuous mapping from the interval $[0, 1]$ into the configuration space. 
%
We define \emph{path obstacle's risk}, denoted by $\texttt{risk}(\pi, O_i)$, as the maximal obstacle risk along the path.
Namely,~$\texttt{risk}(\pi, O_i) := \max_{\alpha \in [0,1]}\{ \texttt{risk}(\pi(\alpha), O_i)\}$. 
To define the risk of a path, we assume that the risk of intersecting the different obstacles is \emph{mutually independent}.
Thus, the risk of a path $\pi$ is defined as~$\texttt{risk}(\pi): = 1- \prod_{O_i \in \calO}(1 - \texttt{risk}(\pi, O_i))$.
Examples of effective shadows and path risk are shown in Fig~\ref{fig:motivation-risk-aggregation}b.

As in the inspection planning problem, here we limit the problem to searching over a roadmap $G =(V,E)$ embedded in $\mathcal{X}$ given start and target vertices $v_s,v_t \in V$ that correspond to $\calR$'s  start and target configurations.
Set $\Pi(v_s, v_t)$ be the set of all paths connecting $v_s$ to $v_t$ in $G$.
The \emph{minimum-risk and path length motion-planning} problem calls for computing the POF of paths in $\Pi(v_s, v_t)$ that are not dominated by any solution with respect to risk and path length.

\ignore{
We can now define the problem we wish to address: 
\begin{prob}[Minimum-risk and path length motion-planning problem]
\label{prob:1}
Let~$\calR$ be a robot operating in a workspace $\calW$ populated by a set $\calO = \{O_1, \ldots O_q\}$ of $q$ obstacles and let $\calX$ denote its configuration space.
Let $x_s, x_t \in \calX$ be start and target configurations, respectively, and let $\Pi(x_s, x_t)$ be the set of all paths connecting $x_s$ to $x_t$.
The \emph{minimum-risk and path length motion-planning} calls for computing the POF made up of every path $\pi^* \in \Pi(x_s, x_t)$ with minimal risk and path length.
Namely,
$\pi^* = 
  \arg \min_{\pi \in \Pi(x_s, x_t)}
    \left(\texttt{risk}(\pi),(|\pi|)\right)$.
\end{prob}

We note that the above formulation is a continuous problem setting, where often motion-planning problems are tackled by sampling a roadmap and then solving a discrete version of the problem. Indeed, this is also the approach we take in our empirical evaluation discussed in Section~\ref{sec:experiments}.
}

In terms of the objective aggregation framework, 
computing all paths in the POF
can be done by relating each obstacle~$O_i$ with a hidden objective $o_i$, and another hidden objective for the path length. Overall, this results in $m=q+1$ hidden objectives.
When extending a path~$\pi$ by an edge $e$, the path extension function 
(i) evaluates the obstacle risk per hidden objective $o_i$ for $i\in\{1,...,q\}$ and (ii) accumulates path length. This is captured by Eq.~\eqref{eq:ext1}.
The solution objectives of minimizing path risk and path length are captured by Eq.~\eqref{eq:agg1}.

\begin{figure*}[t]
    \centering
    \captionsetup[subfigure]{justification=centering,singlelinecheck=false}

    \begin{subfigure}[t]{0.24\textwidth}
        \centering
        \includegraphics[height=0.11\textheight]{\figType{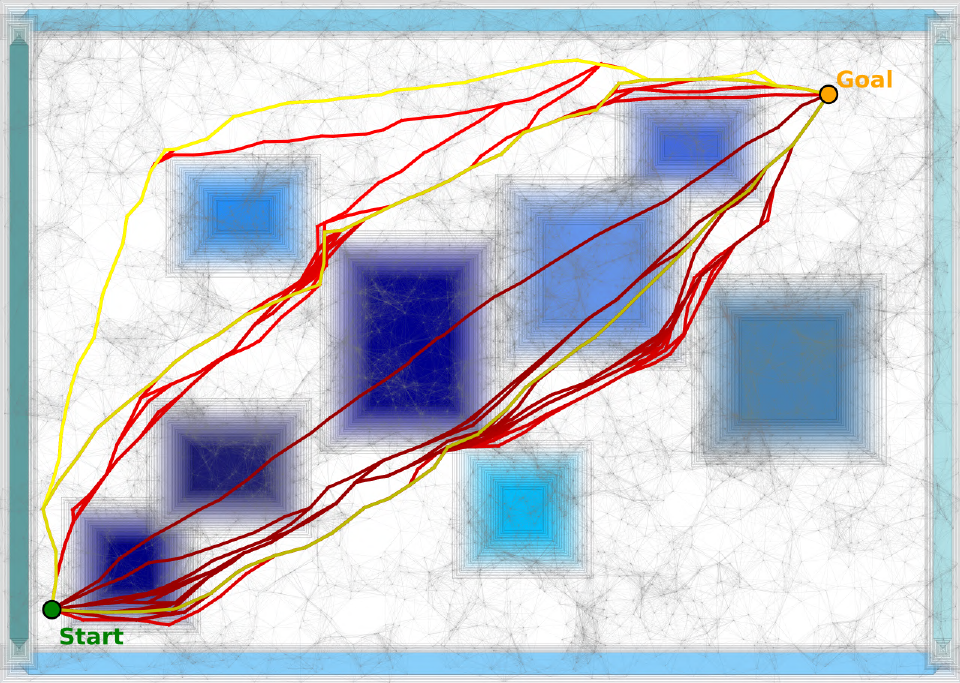}{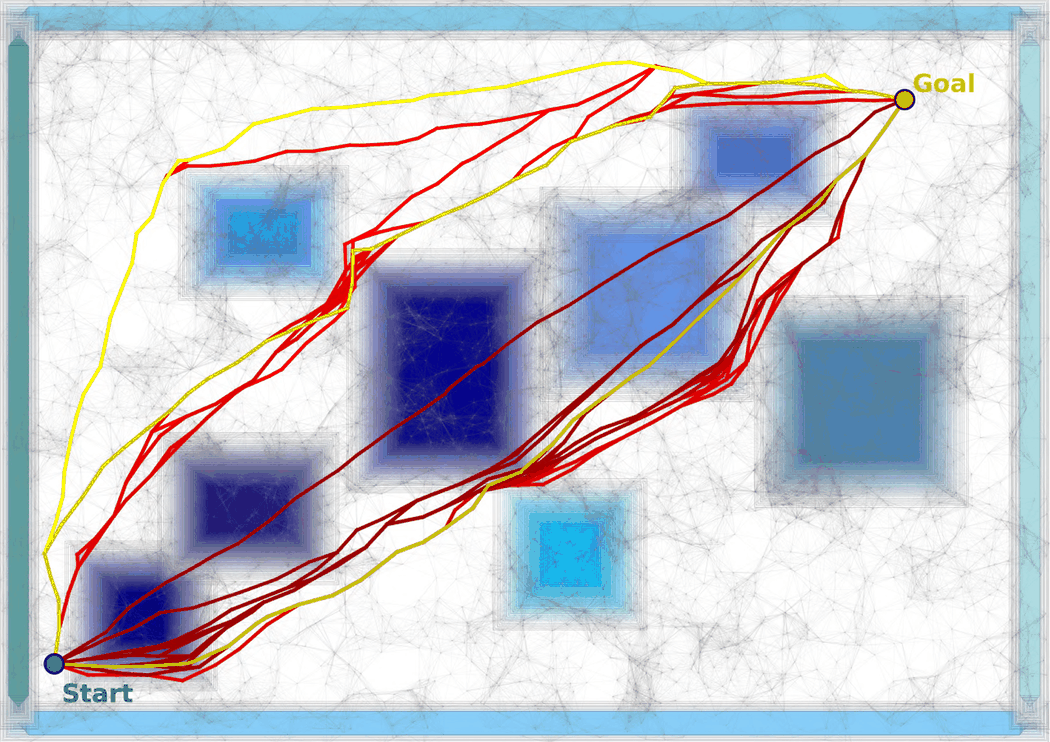}}
        \caption{}\label{Fig_point_env}
    \end{subfigure}
    \hfill
    \begin{subfigure}[t]{0.24\textwidth}
        \centering
        \includegraphics[height=0.11\textheight]{\figType{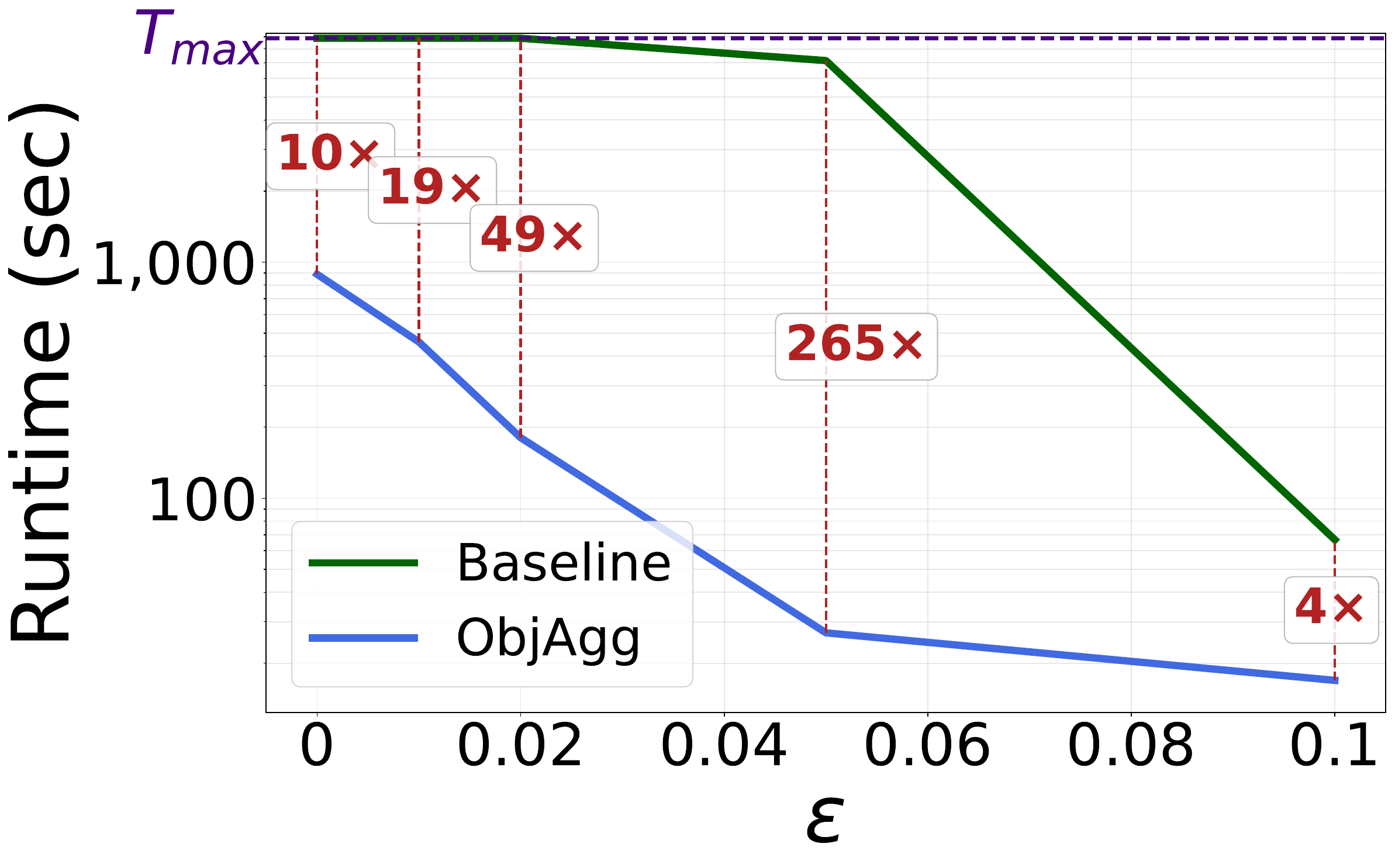}{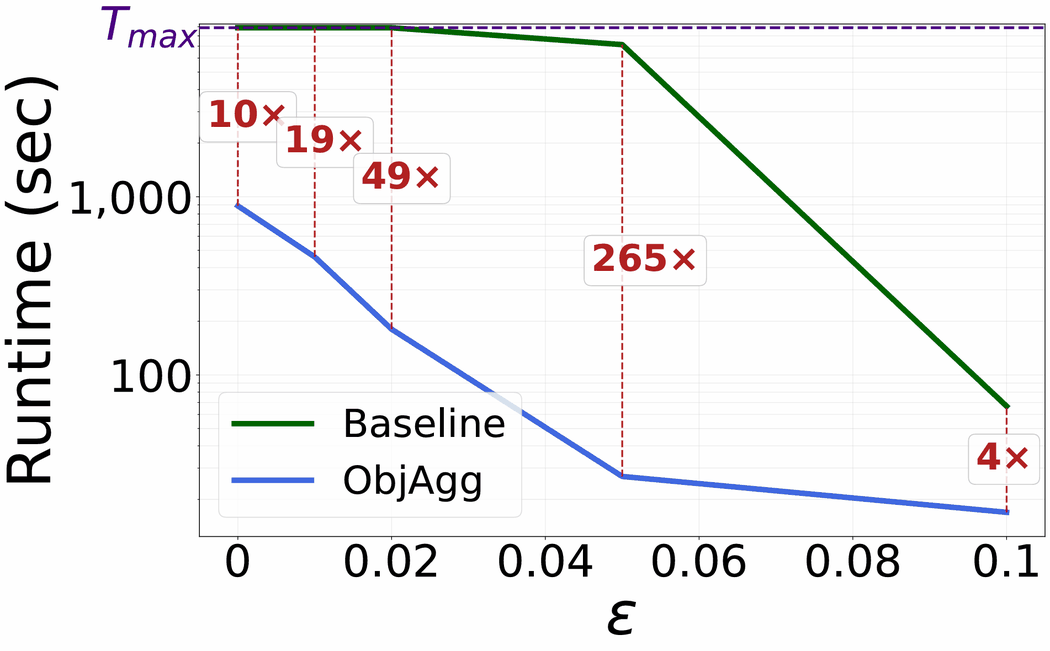}}
        \caption{}
        \label{fig:speedup-runtime_base}
    \end{subfigure}
    \hfill
    \begin{subfigure}[t]{0.24\textwidth}
        \centering
        \includegraphics[height=0.11\textheight]{\figType{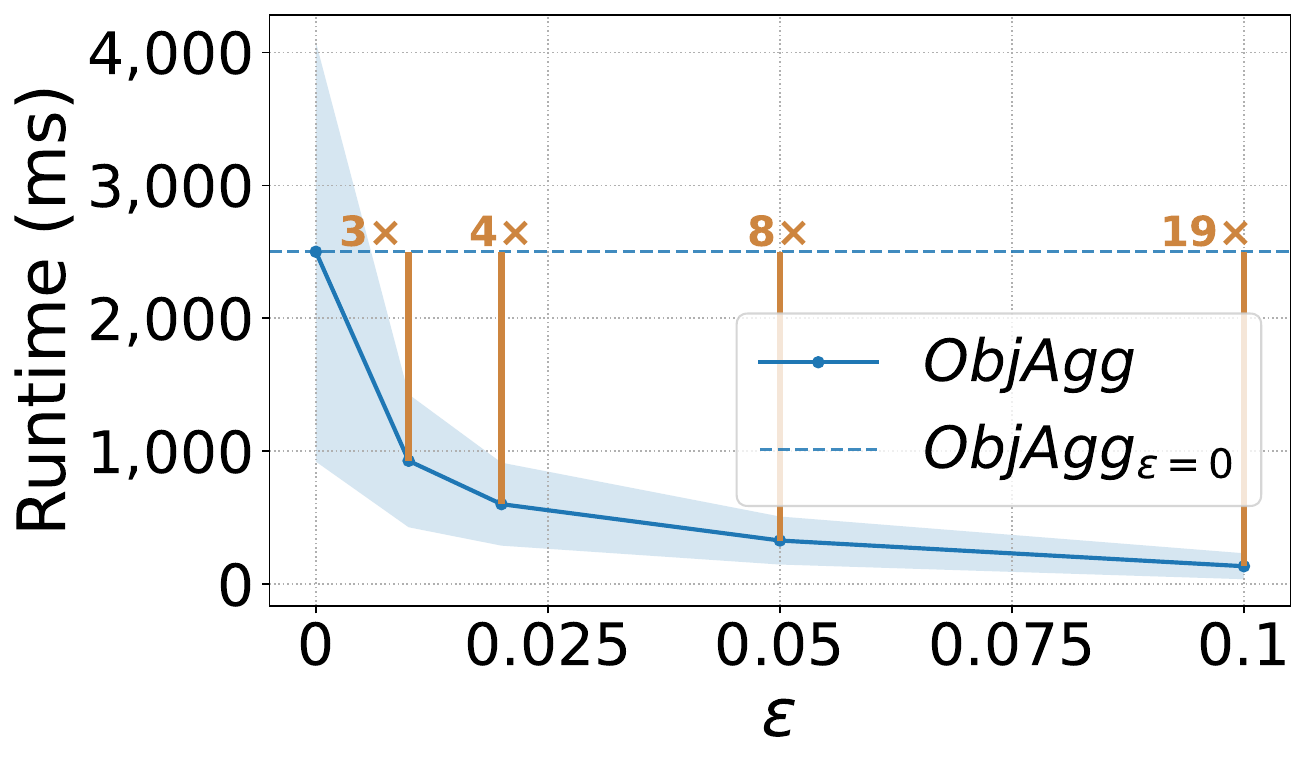}{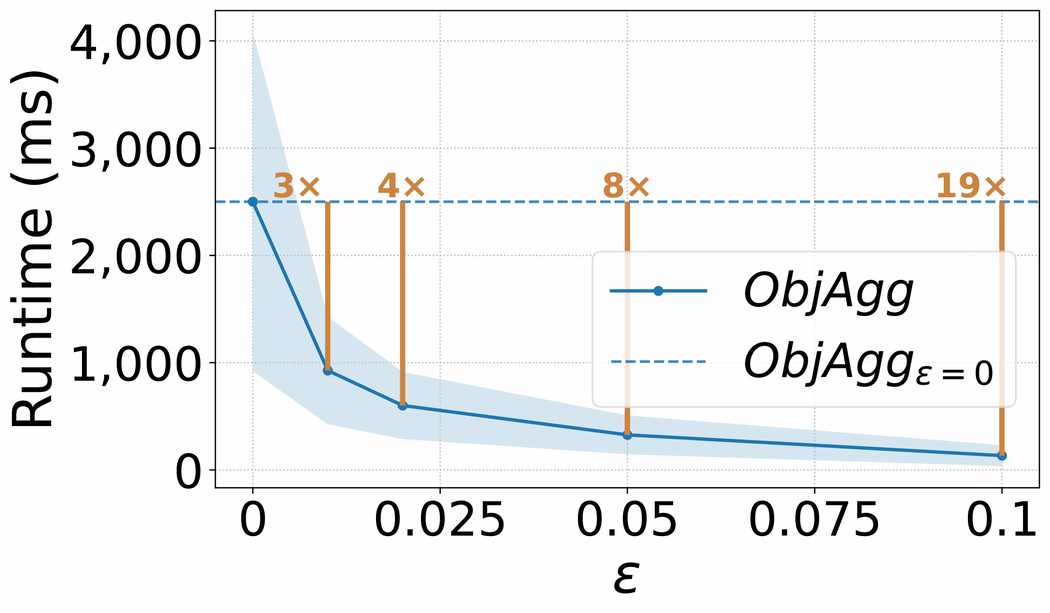}}
        \caption{}
        \label{fig:speedup-runtime_eps}
    \end{subfigure}
    \hfill
    \begin{subfigure}[t]{0.24\textwidth}
        \centering
        \includegraphics[height=0.11\textheight]{\figType{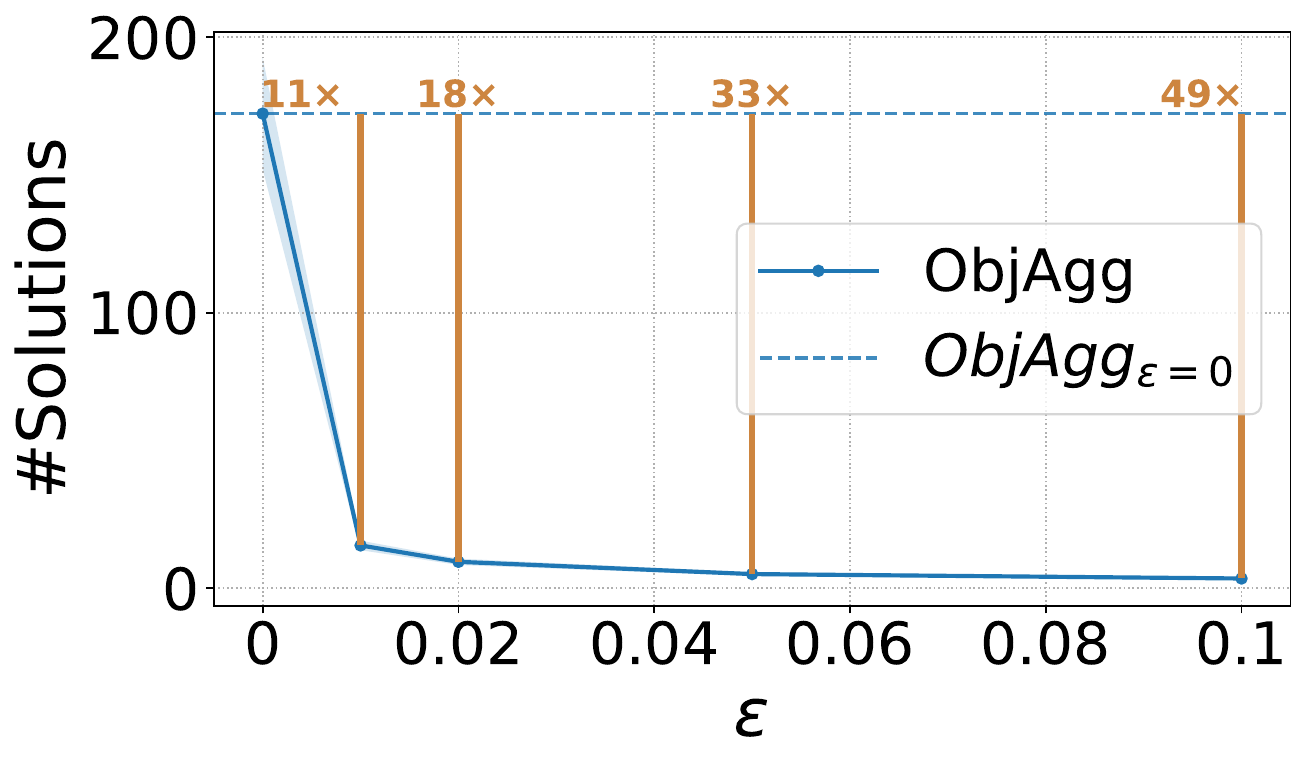}{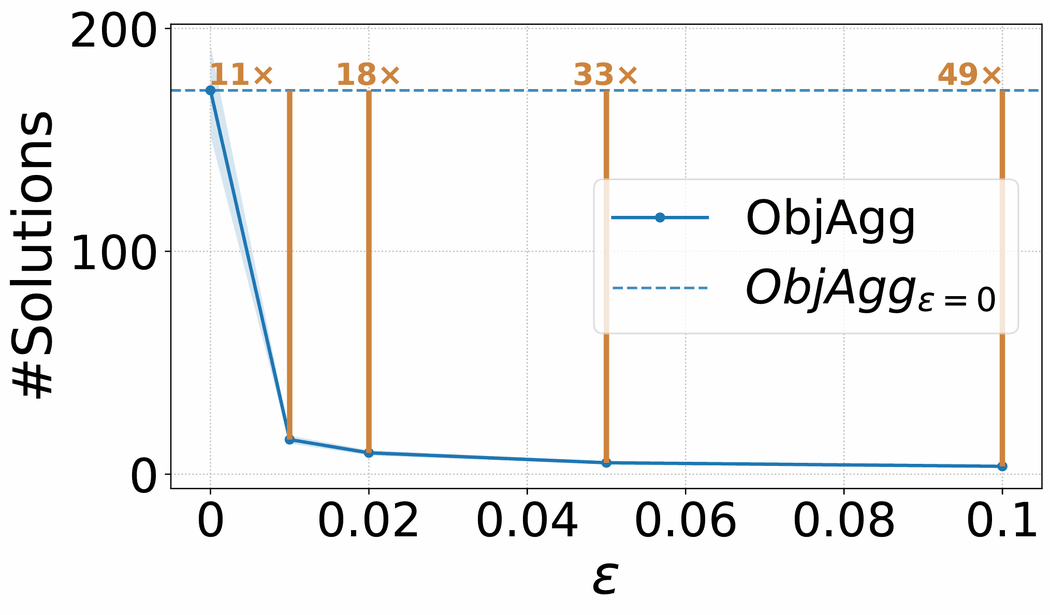}}
        \caption{}
        \label{fig:speedup-solutions}
    \end{subfigure}
    \caption{
    (a)~Point-robot environment together with the roadmap (gray), obstacles with shadows (blue) and one start-goal pair along with all paths in $\Pi^*$.
    (b-d)~Performance as a function of the approximation factor~$\varepsilon$, average speedups indicated.
    (b)~Runtimes of \algname{Baseline} and \algname{ObjAgg} shown in log-scale, speedups in red.
    (c)~Runtimes of \algname{ObjAgg} for changing~$\varepsilon$ and fixed~$\varepsilon=0$, speedups in brown. 
    (d)~Number of solutions of \algname{ObjAgg} for changing~$\varepsilon$ and fixed~$\varepsilon=0$, speedups in brown.
    }
    \vspace{-7pt}
\end{figure*}

\subsection{Route Planning with Road Types}
\label{subsec:route}
In route planning with road types, 
we are given a ground vehicle $\R$,
a road network represented as a graph $G$
with two types of edges corresponding to paved and unpaved roads.
Following state-of-the-art navigation applications such as Waze, we are tasked with planning the shortest route while avoiding long unpaved roads.
Rephrased as an optimization problem, we aim to minimize both the overall path length and the longest partial path on unpaved roads.

In terms of the objective aggregation framework, this problem can be addressed by defining three hidden objectives:
$g$ for the accumulated road length, $g_{\text{con}}$ for the consecutive unpaved road length, and $g_{\text{max}}$ for the maximal consecutive unpaved road length.
The solution objectives in this case are a subset of the hidden objectives: $g$ and $g_{\text{max}}$.
The information regarding an edge $e$ can be encoded by the cost function $\mathbf{c}$, which returns a tuple $(\ell,\texttt{Type})$, where the former is the road length and the latter is equal to $1$ for paved and $0$ for unpaved roads.
Thus, we obtain the following extension and aggregation functions:
%
\begin{equation}
\mathcal{F}_{\rm ext}\big((g,g_{\text{con}},g_{\text{max}}),(\ell,\texttt{Type}) \big) = 
{\color{white}.\hspace{24mm}.}
\end{equation}
\[
\begin{cases}
\big(g+\ell,\; g_{\text{con}}+\ell,\; \max(g_{\text{con}}+\ell, g_{\text{max}})\big) & \quad \text{if } \texttt{Type}=0, \\[6pt]
\big(g+\ell,\; 0,\; g_{\text{max}}\big) & \quad \text{if } \texttt{Type}=1.
\end{cases}
\]

\begin{equation}
\label{eq:agg3}
    \mathcal{F}_{\rm agg} (g,g_{\text{con}},g_{\text{max}}) = \left(g,g_{\text{max}} \right).
\end{equation}

Notice that here~$m=3, d=k=2$, and interestingly~$\mathcal{F}_{\rm ext}$ is only monotonic w.r.t. $g,g_{\text{max}}$ (a paved edge resets $g_{\text{con}}$ to $0$); Nevertheless, monotonicity is still preserved for the solution objectives. In the empirical evaluation (presented in Sec.~\ref{sec:experiments}) we will discuss practical implications of such cases.

\section{Experiments}
\label{sec:experiments}

We empirically demonstrate the efficacy of using objective aggregation, specifically in the context of computing approximations of $\Pi^*$ when compared to unmodified MOS algorithms.
We evaluate our approach on four different problems 
which simulate 
planning under obstacle uncertainty (Sec.~\ref{subsec:shadows})
and 
route planning with road types (Sec.~\ref{subsec:route}).
For the former we consider 
(i)~a point robot moving in the plane,
(ii)~a $5$-link planar manipulator 
and
(iii)~The Continuum Reconfigurable Incisionless Surgical Parallel (CRISP)~\cite{mahoney2016reconfigurable} robot
moving in the inner surface of a pleural cavity.

We use \algname{NAMOA-dr} as our MOS  algorithm \algname{ALG}.\footnote{Importantly, we also evaluated objective aggregation using the \apex algorithms. Results ommitted in favor of simplicity of the presentation. The implementation will still be available in our open-source repository.}
For simplicity, we use 
\algname{Baseline} and \algname{ObjAgg} to refer to the straw man algorithm described in Sec.~\ref{subsec:gen_mos_algorithms} 
(i.e., to 
$(m,m)$-$\algname{ALG}_\varepsilon$)
and to our proposed objective aggregation framework
(i.e., to 
$(k,m)$-$\algname{ALG}_\varepsilon$ for $k<m$).
Algorithms were implemented in C\raise.08ex\hbox{\tt ++}\xspace, and experiments were conducted on Dell Latitude 5411 with 16 GB RAM, Intel i7-10850H (2.70 GHz), and Ubuntu 18.04.5 LTS. Code, experimental setup and environments will be made publicly available upon paper acceptance.

In all experiments, we use an approximation factor of~$\varepsilon\geq 0$ and applied the graph-distance heuristic (i.e., shortest path distances) w.r.t. the path-length objective (or Euclidean distance in C-space in the manipulator environments).


\paragraph{Navigation under OU -- Point Robot}

\begin{figure}[t]
    \centering
    \begin{subfigure}[t]{0.45\columnwidth}
        \centering
        \includegraphics[width=0.9\linewidth]{\figType{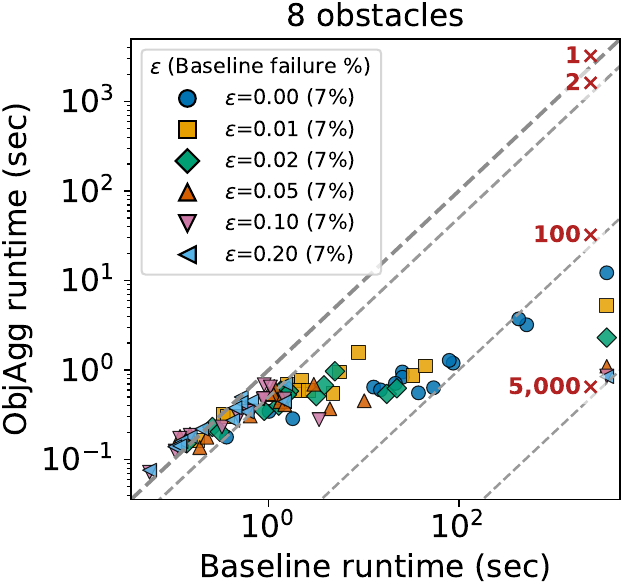}{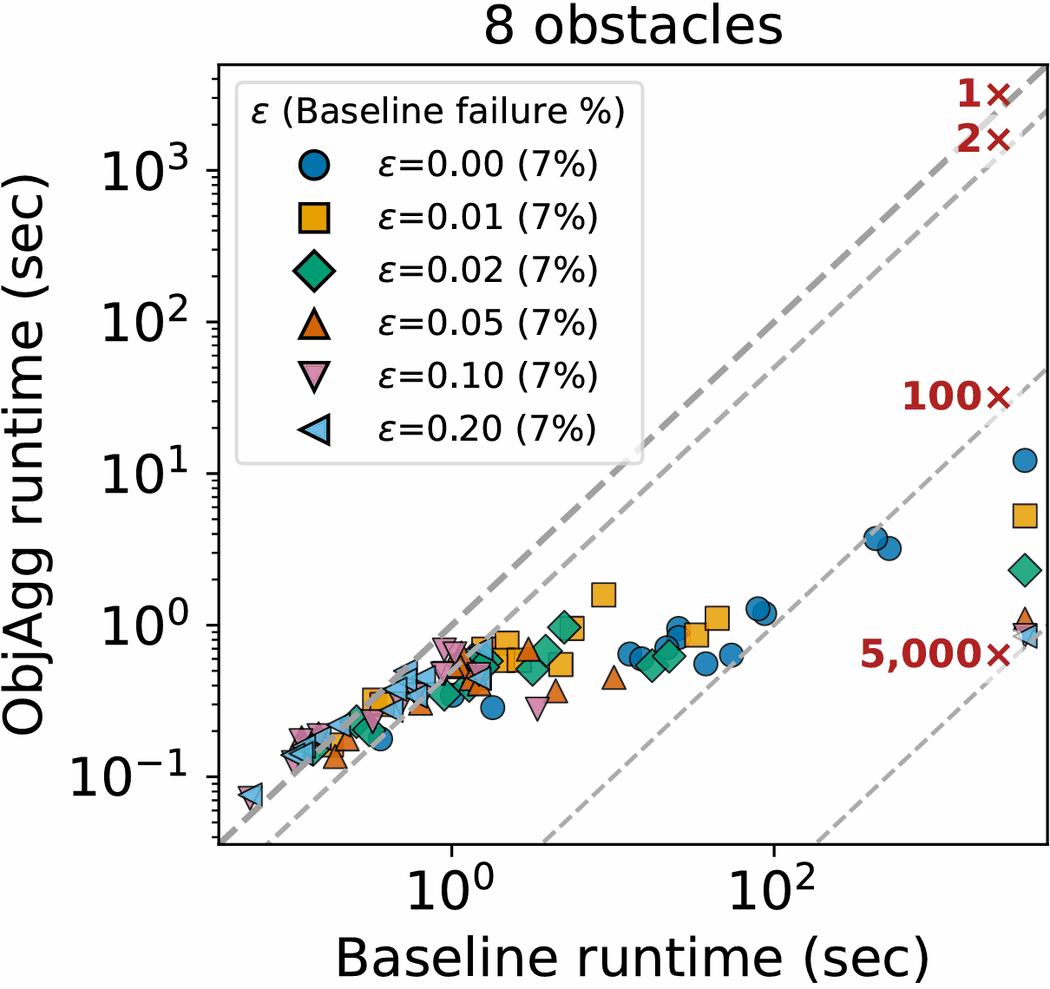}}
        \caption{}
        \label{fig:boa-vs-baseline-8}
    \end{subfigure}
     \hfill
    \begin{subfigure}[t]{0.45\columnwidth}
        \centering
        \includegraphics[width=0.9\linewidth]{\figType{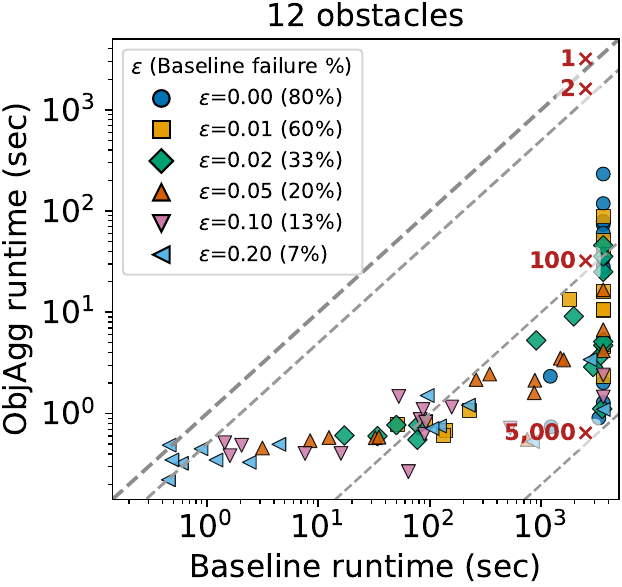}{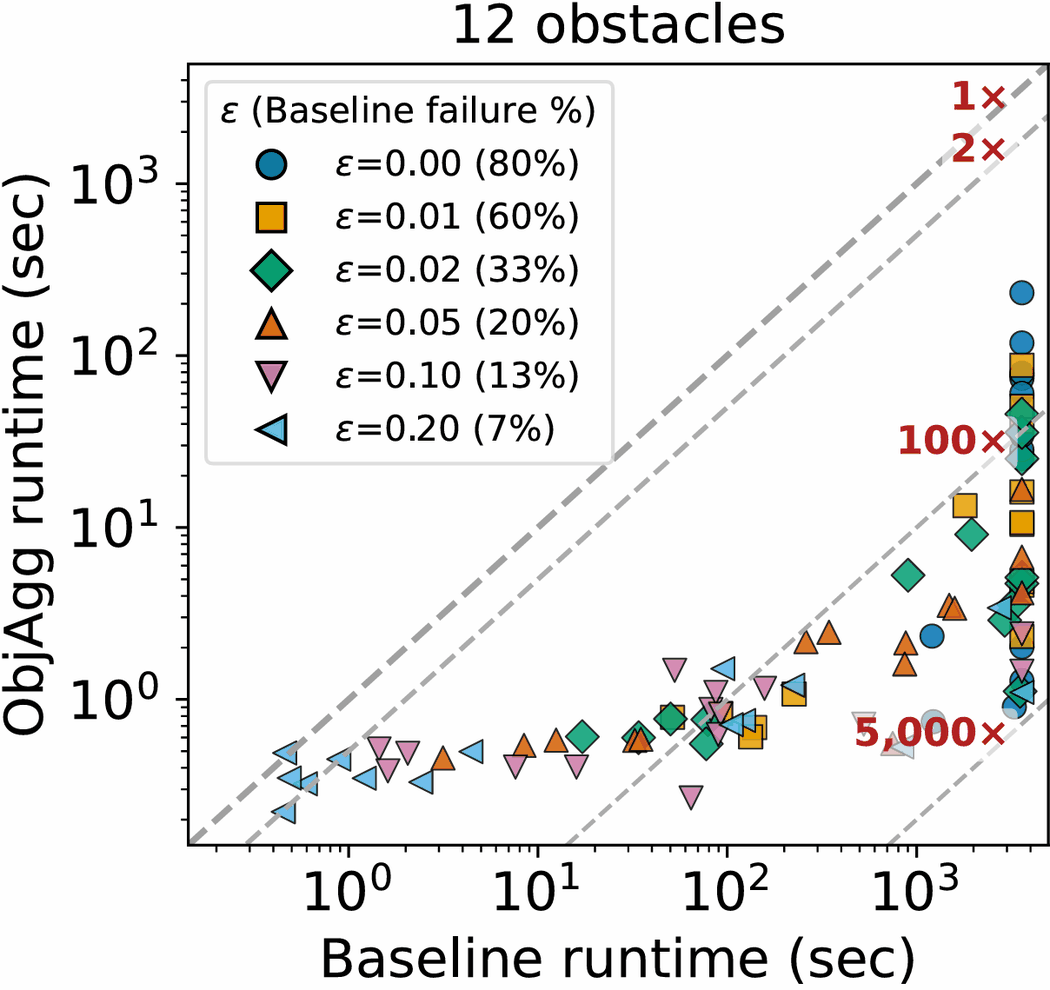}}
        \caption{}
        \label{fig:boa-vs-baseline-12}
    \end{subfigure}

    \caption{
    Point robot log–log scatter plots of runtime (seconds) 
    for \algname{Baseline} (horizontal axis) and \algname{ObjAgg} (vertical axis), 
    for environments with $8,12$ obstacles. 
    Each marker corresponds to a scenario at a given $\varepsilon_{\text{length}}$ value, with consistent color and marker mapping across panels. 
    The diagonal $1\times$ line indicates equal runtime, and dashed guide lines show speedup factors of $2\times$, $5\times$, $100\times$, and $5000\times$. 
    Legends list $\varepsilon$ values for the path length objective together with the percentage of \algname{Baseline} runs that reached the time limit~$T_{\max}$  (failure rate). 
    }
    \label{fig:boa-vs-baseline}
    \vspace{-4mm}
\end{figure}

\begin{figure*}[h]
  \centering
  \hspace*{\fill}
     \begin{minipage}[t]{0.317\textwidth}
    \centering
    \vspace{0pt}
    \subcaptionbox{}{%
      \includegraphics[width=0.9\linewidth]{\figType{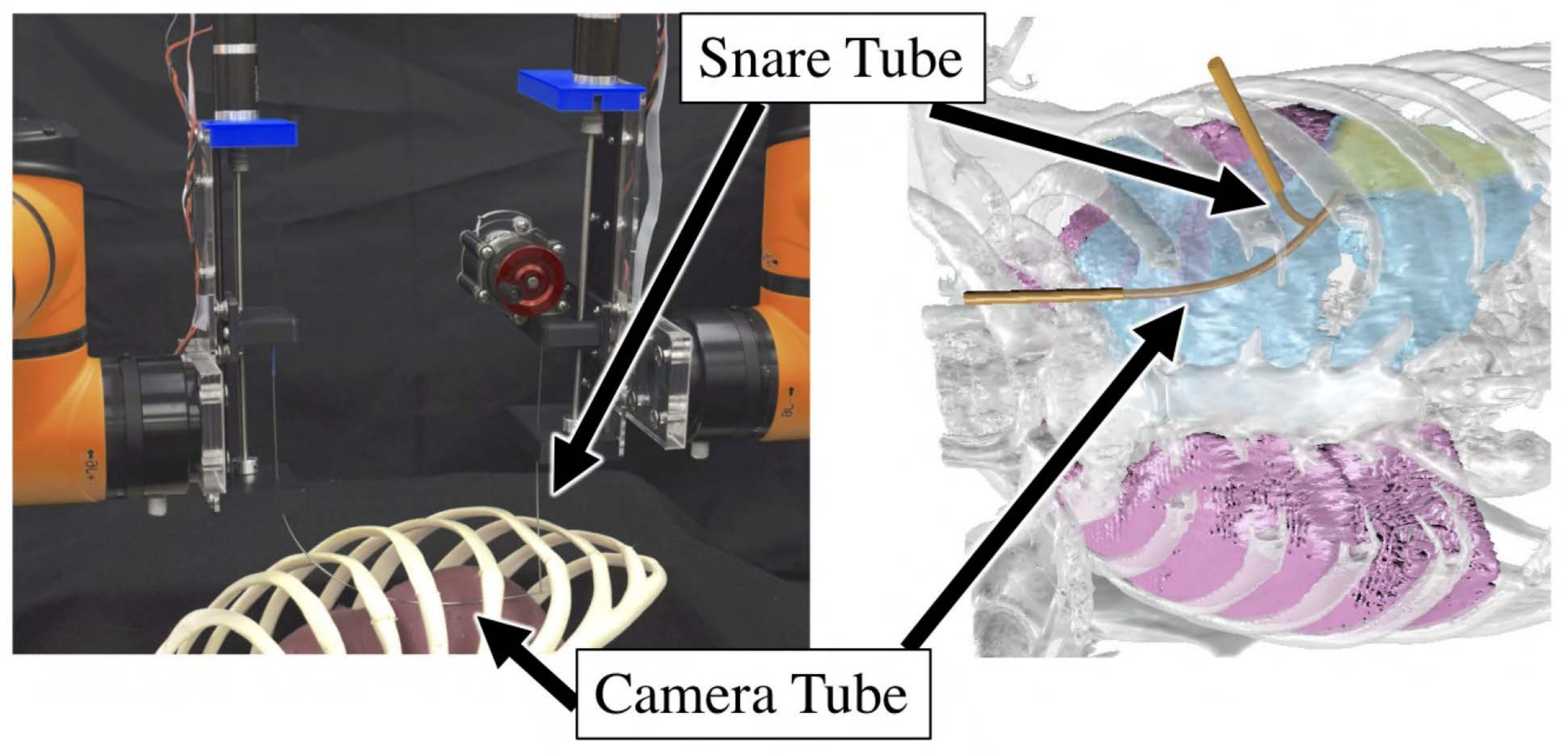}{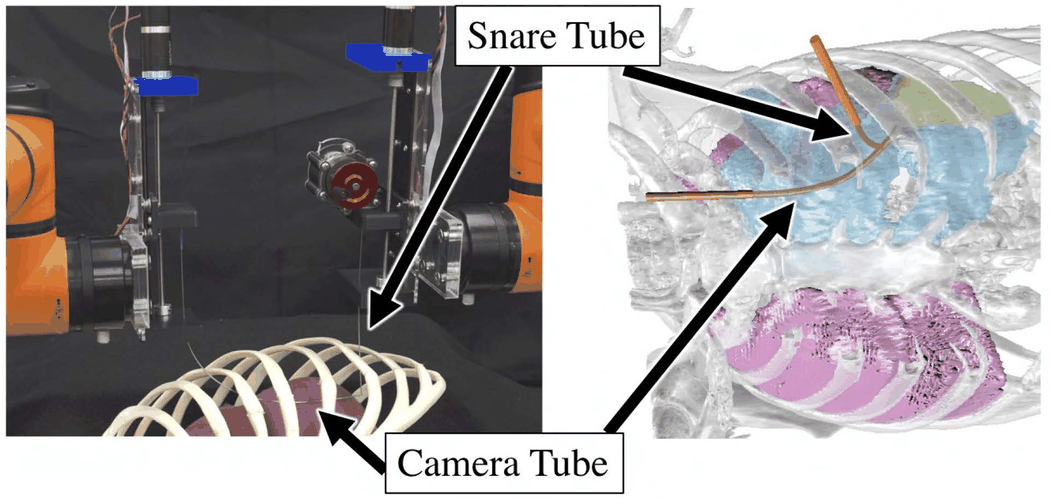}}%
    }\\[0.5em]
    \begin{subfigure}[t]{0.49\linewidth}
      \centering
      \includegraphics[width=\linewidth]{\figType{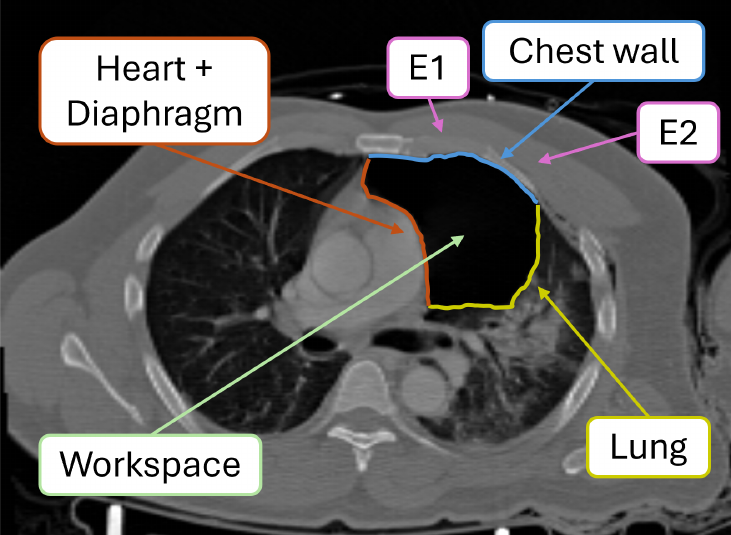}{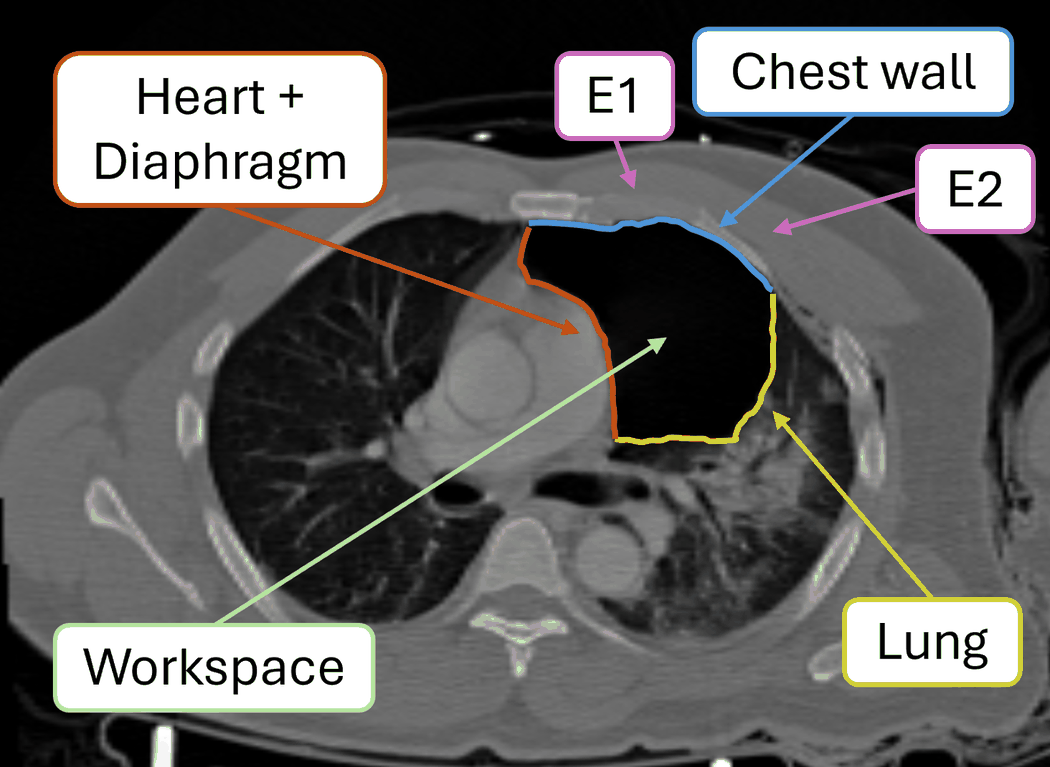}}
      \caption{}\label{fig:crisp-e}
    \end{subfigure}
    \begin{subfigure}[t]{0.49\linewidth}
      \centering
      \includegraphics[width=\linewidth]{\figType{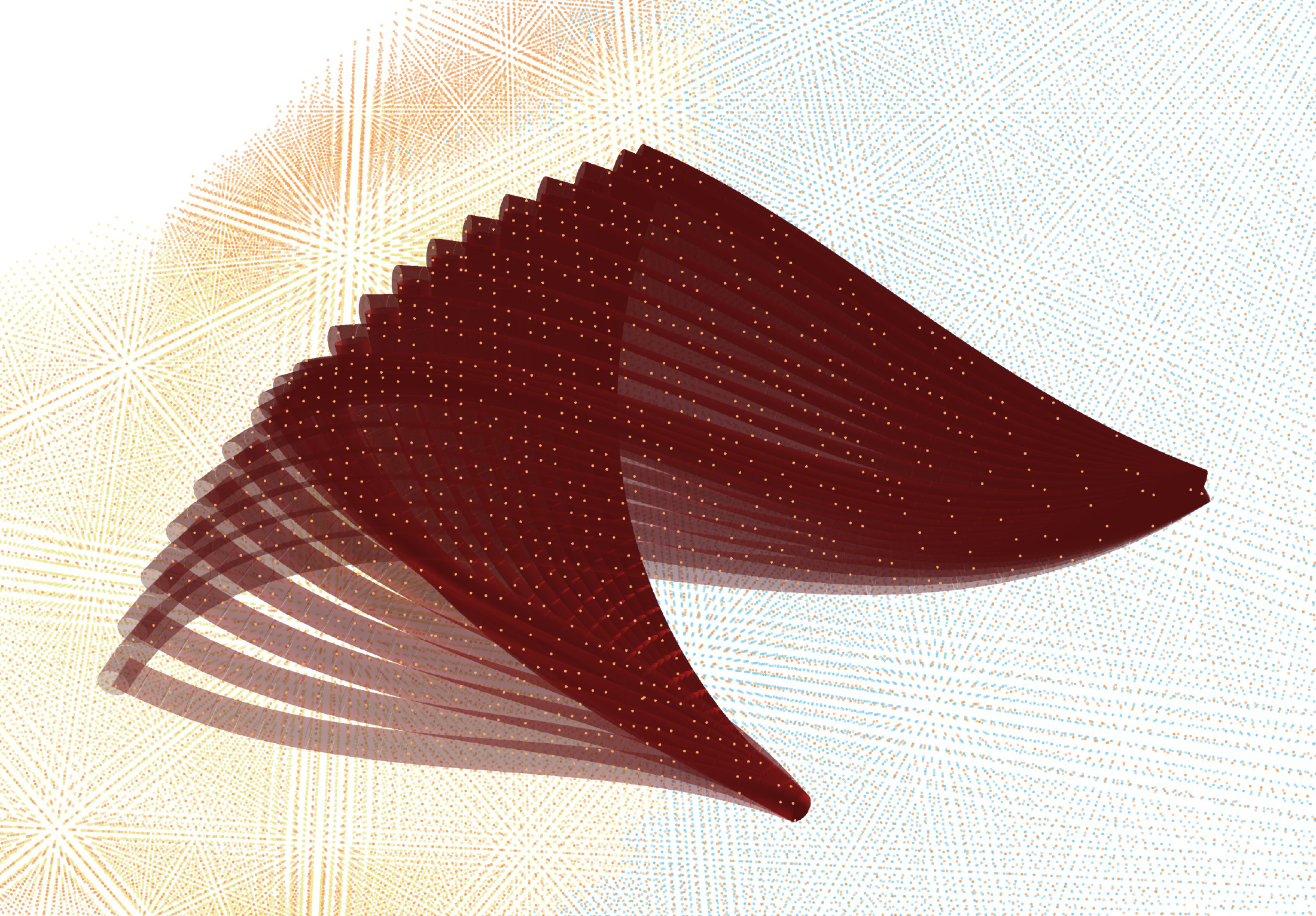}{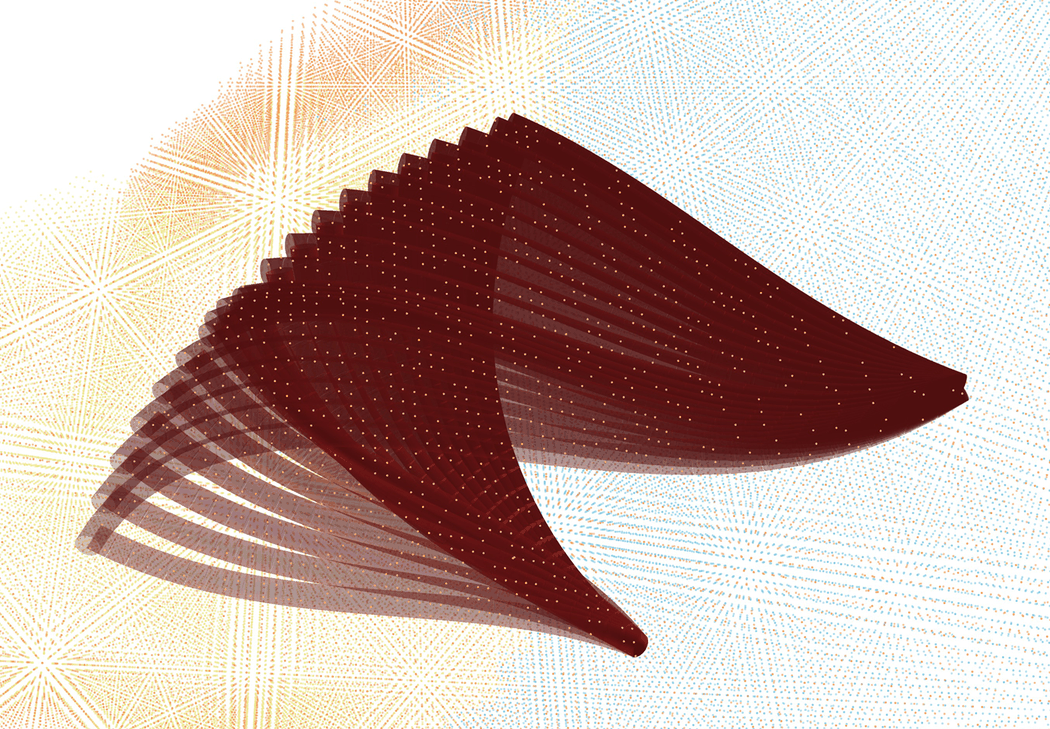}} 
      \caption{}\label{fig:crisp-f}
    \end{subfigure}
  \end{minipage}
  \hfill
\begin{minipage}[t]{0.1865\textwidth}
    \centering
    \vspace{0pt}
    \subcaptionbox{}{%
      \includegraphics[width=\linewidth]{\figType{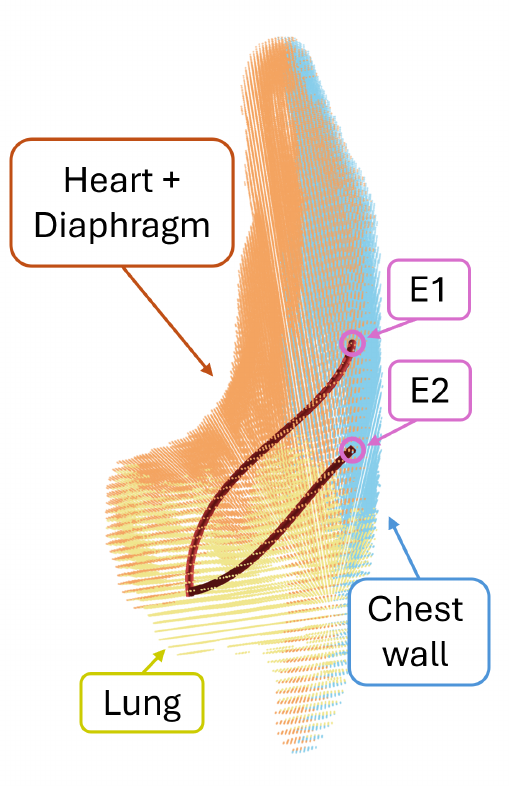}{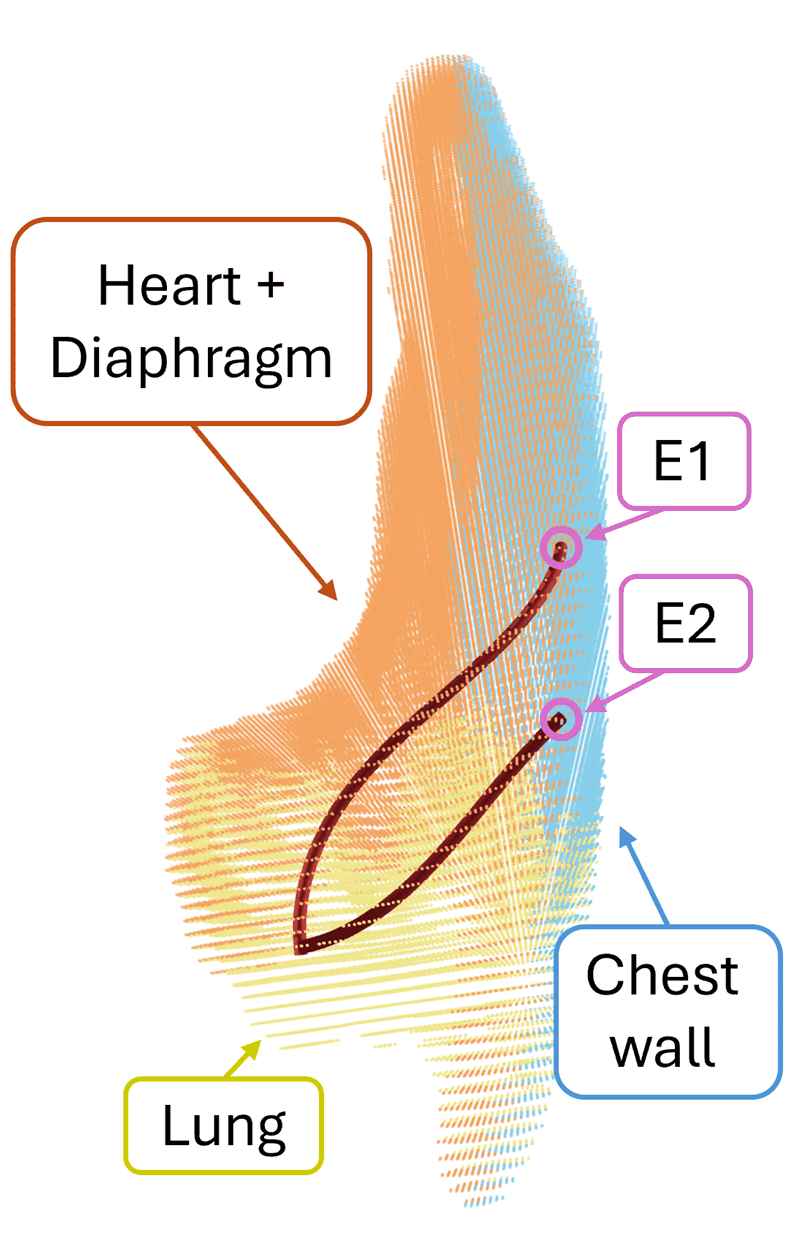}}%
    }
  \end{minipage}
    \hfill
  \begin{minipage}[t]{0.194\textwidth}
    \centering
    \vspace{0pt}
    \subcaptionbox{}{%
      \includegraphics[width=\linewidth]{\figType{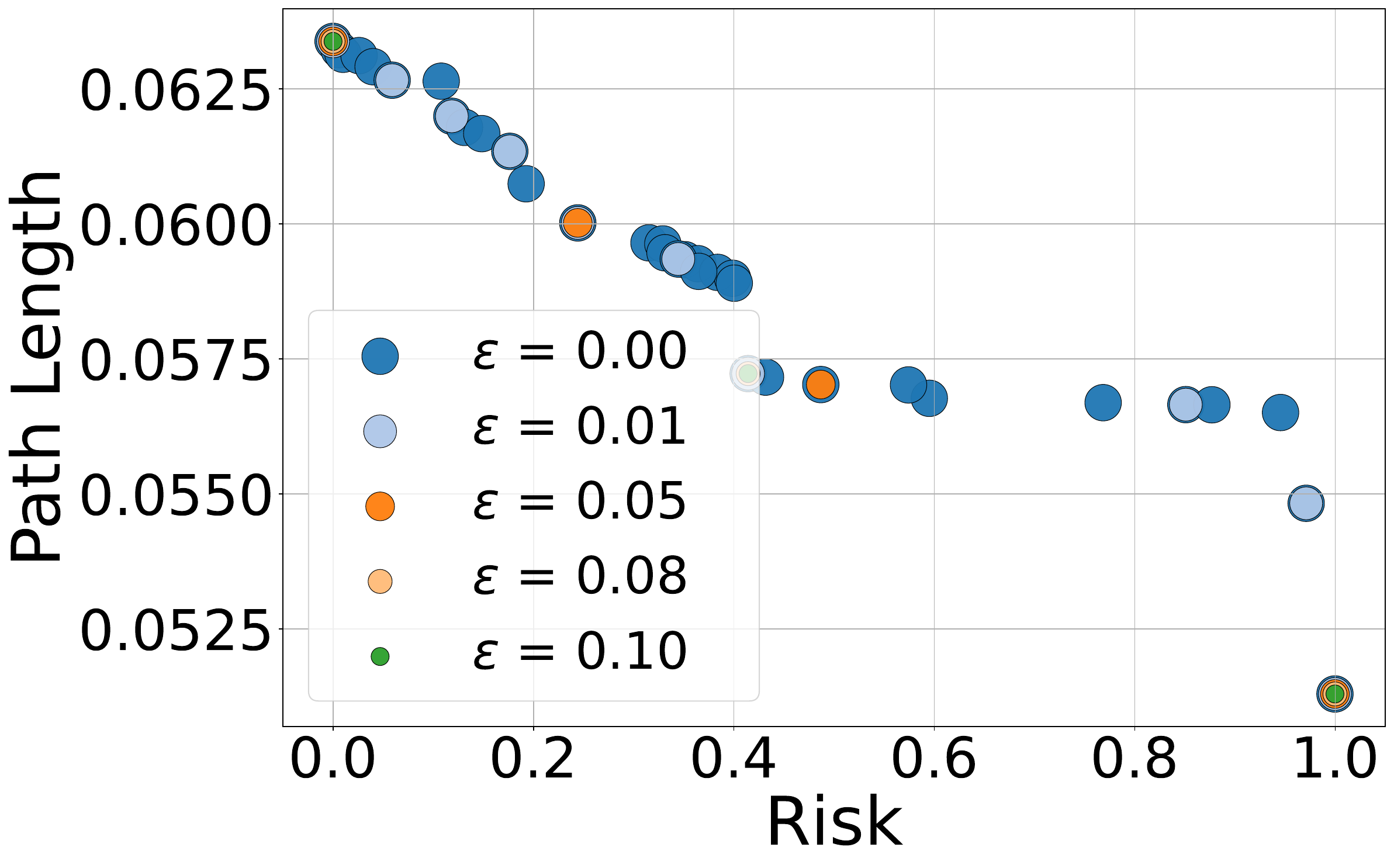}{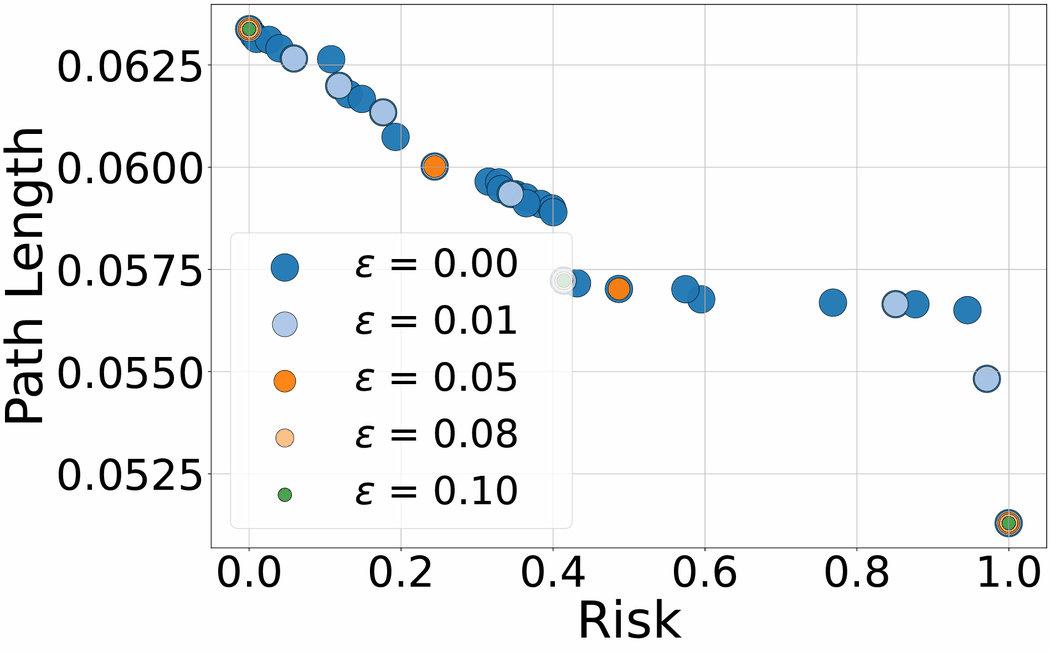}}%
    }\\[1em]
    \subcaptionbox{}{%
      \includegraphics[width=\linewidth]{\figType{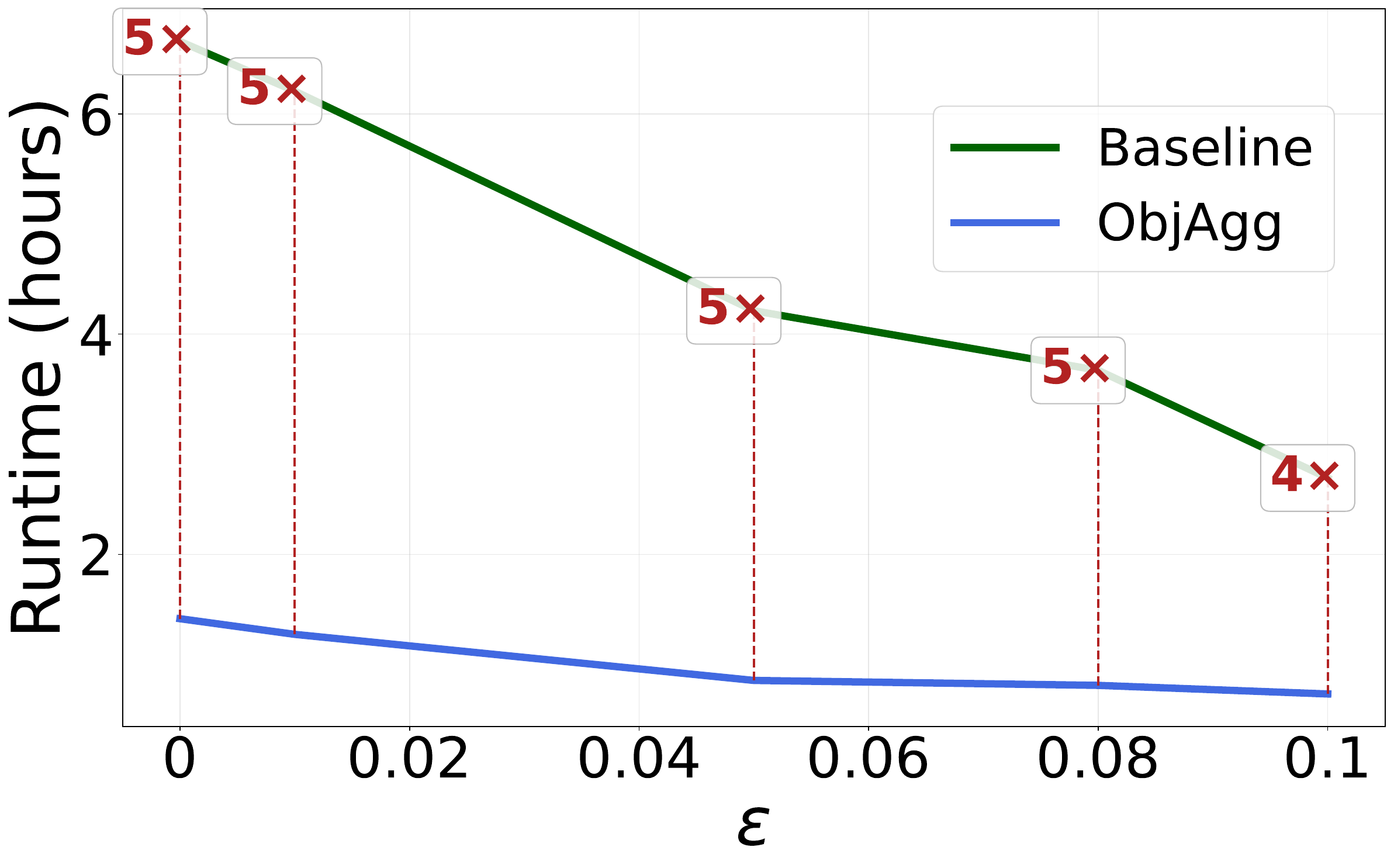}{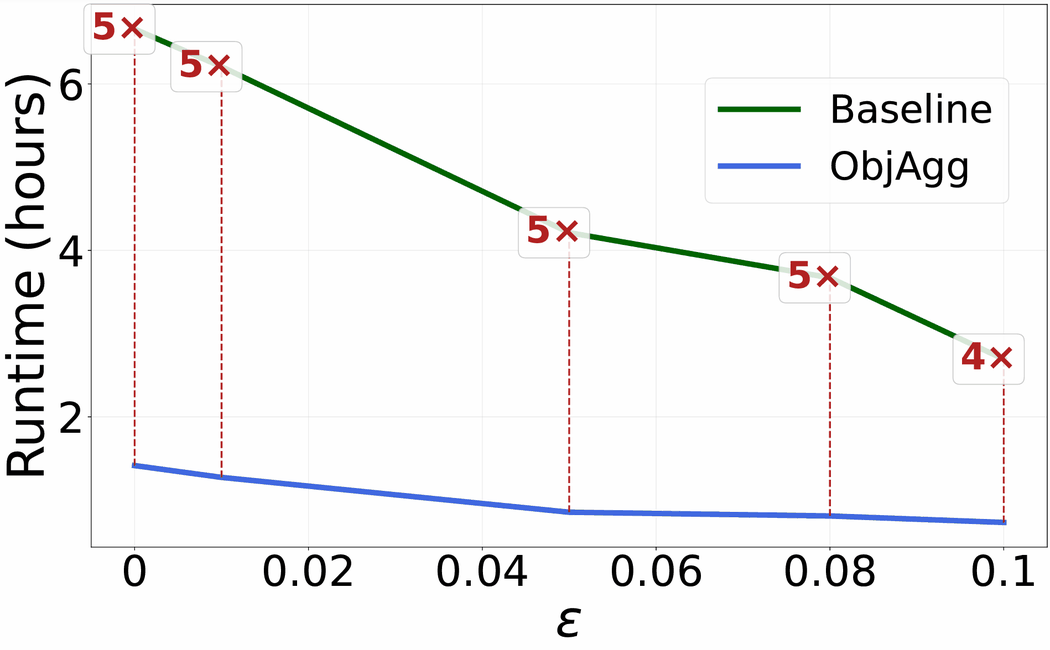}}%
    }
  \end{minipage}
  \hspace*{\fill}
\caption{CRISP robot operating in a lung: (a) Segmented anatomical surfaces with entry points E1 and E2. (b)~CT slice visualization. (c) Example of the  motion of the CRISP's tubes. (d) CRISP robotic platform (figure adapted from~\cite{FuKSA23}). (e)~Approximate Pareto-optimal frontiers for different $\varepsilon$ values. (f) Runtime of \algname{ObjAgg} and \algname{Baseline} as a function of~$\varepsilon$.}
  \label{fig:crisp-all}
  \vspace{-6pt}
\end{figure*}

We consider a point robot navigating a 2D environment cluttered with rectangular obstacles. 
Uncertainty is modeled using Gaussian distributions over the obstacles’ widths and heights.
For each obstacle, we generated 30 equally-spaced shadows, with each shadow increasing in size by a constant amount. The risk probability of each shadow is evaluated using the corresponding Gaussian distribution. In a preprocessing phase, we generated a roadmap using the PRM algorithm~\cite{prm} and sampled $20$ pairs of vertices whose distance exceeds a minimal threshold, of $0.6\cdot d_{\rm env}$, where $d_{\rm env}$ is the distance between opposite corners in the environment. 
For each query consisting of a start-goal pair and a value of $\varepsilon$, we set a timeout of $T_{\max} = 3600$ seconds. 
Visualization of the environment and a subset of the POF are provided in Fig.~\ref{Fig_point_env}.
Here, we have $m=13$ hidden objectives, reflecting 12 obstacles and the path length.
The experimental results, summarized in Fig.~\ref{fig:speedup-runtime_base}-\ref{fig:speedup-solutions}, demonstrate that \algname{ObjAgg} significantly outperforms \algname{Baseline} in terms of runtime, a trend which further intensifies with larger $\varepsilon$ values.
Additionally, when~$m$ (number of hidden objectives) increases, the advantage of our aggregation-based method becomes more pronounced, achieving speedups of up to $5000\times$ (see Fig.~\ref{fig:boa-vs-baseline}).

\paragraph{Manipulation under OU -- Planar Manipulator}
We consider a 5-link manipulator operating in 2D. The setup is motivated by the task of grasping a milk carton on a cluttered refrigerator shelf with rectangular obstacles modeled as boxes (see Fig.~\ref{fig:planar_env}). The environment contains 12 obstacles, corresponding to 12 hidden objectives. Uncertainty is represented as Gaussian noise over the obstacle locations. For each
of the~12 obstacles, we generated 30 equally-spaced shadows, with each shadow increasing in size by a constant amount. The risk of each shadow is evaluated using the corresponding Gaussian distribution. In a preprocessing phase, a roadmap is generated using the PRM algorithm. We demonstrate the evaluation on a single start–goal configuration task (see Fig.~\ref{fig:planar_env}). For this instance, we set a timeout of~$T_{\max}=1500$~sec. Our algorithm \algname{ObjAgg} required $152$~sec when $\varepsilon=0$ and $80$~sec when $\varepsilon=0.2$, while \algname{Baseline} reached the timeout $T_{\max}$ for all values of $\varepsilon$, reflecting speedups of at least $10-19\times$.

\begin{figure}[t]
  \centering
  \begin{subfigure}[t]{0.44\columnwidth}
    \centering
    \includegraphics[width=\linewidth]{\figType{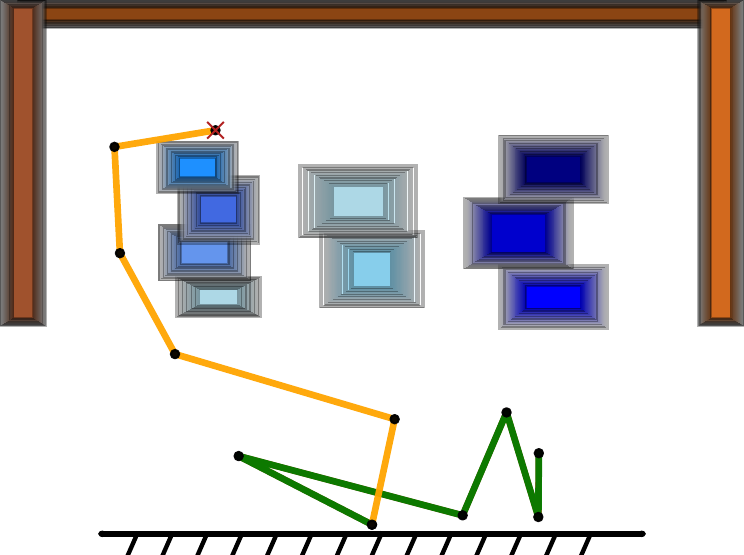}{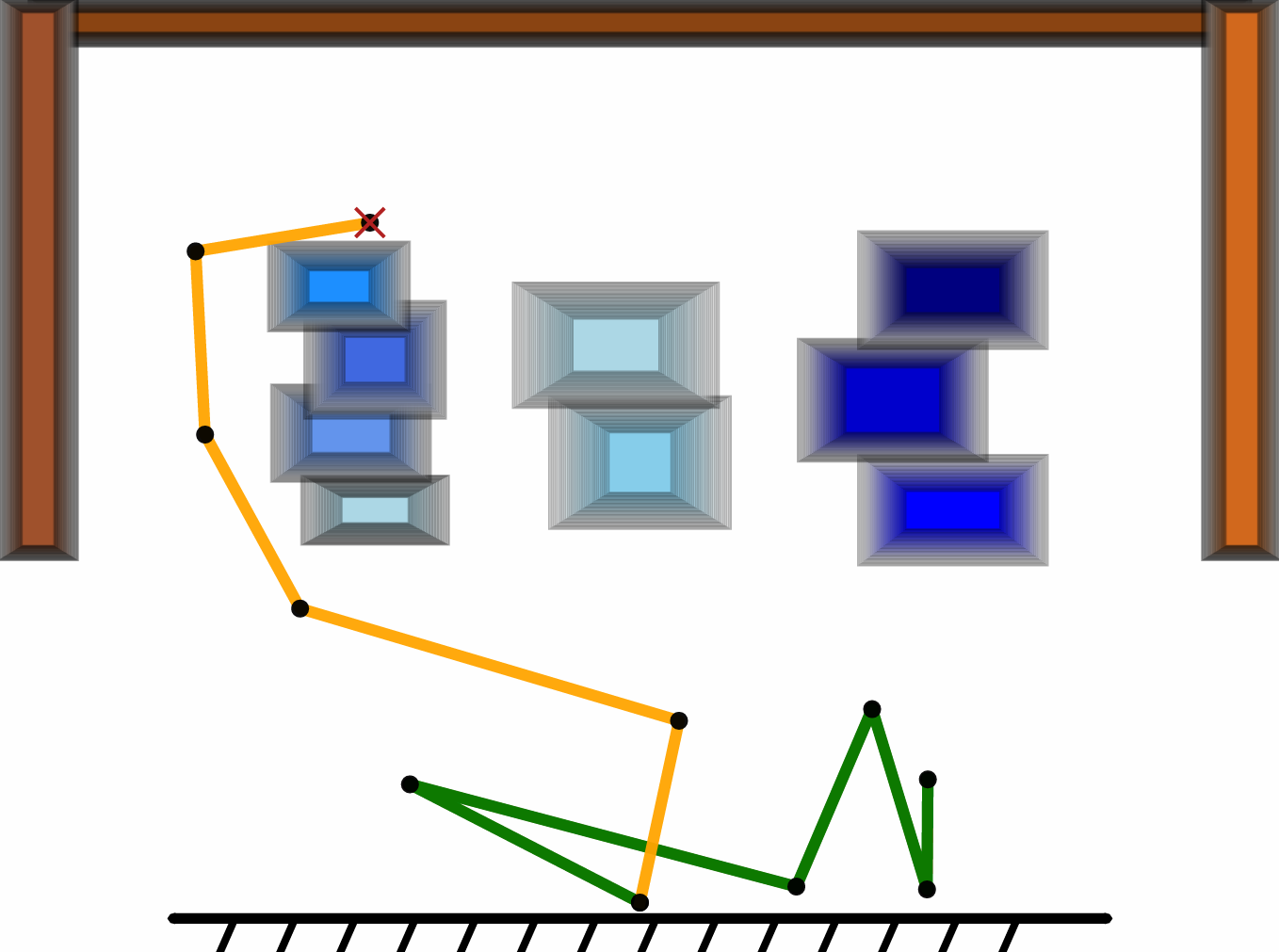}}
    \caption{}
    \label{fig:planar_env}
  \end{subfigure}\hfill
  \begin{subfigure}[t]{0.52\columnwidth}
    \centering
    \includegraphics[width=\linewidth]{\figType{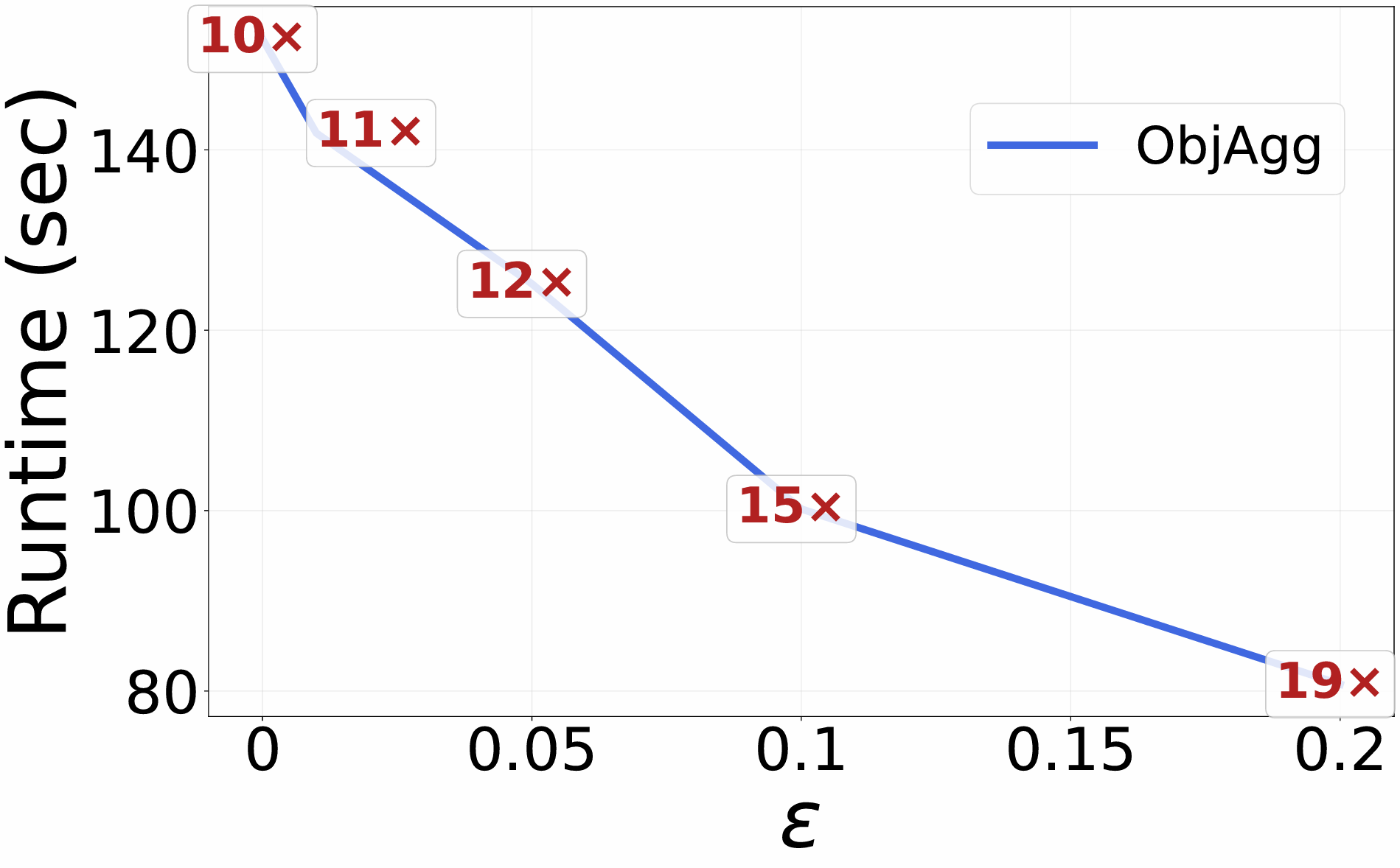}{img/planar/planar_speedup.png}}
    \caption{}
  \end{subfigure}
 \caption{(a) 5-link manipulator operating in a refrigerator shelf cluttered with obstacles. Initial configuration in green, terminal configuration in orange. (b) Runtime of \algname{ObjAgg} as a function of increasing $\varepsilon$ and speedups compared to \algname{Baseline}.}
  \label{fig:planar-all}
  \vspace{-10pt}
\end{figure}

\paragraph{Manipulation under OU -- CRISP robot}
Pleural effusion is a medical condition in which excess fluid accumulates in the pleural cavity, the space between the lungs and the chest wall. To evaluate our method in a clinically motivated setting, we used an anatomical pleural effusion environment derived from a Computed Tomography (CT) scan of a real patient diagnosed with the condition. The internal surface of the cavity is represented as a point cloud, where each point is treated as an obstacle assumed to exist with full certainty. A shadow volume includes all points within a specified distance from an obstacle, and the collision probability with an obstacle is estimated by identifying its innermost shadow, i.e., the smallest distance between the robot and the corresponding obstacle points. Under this uncertainty model, we account for the possibility that the free space available to the robot may shrink, reflecting the uncertainty in pleural effusion anatomy between CT scan acquisition and the medical procedure. For analysis, we segmented the internal surface into three distinct obstacle regions: the lung, the heart and diaphragm (as one obstacle), and the chest wall (excluding the tube entry points). Therefore, in these settings the number of hidden objectives is $m=4$.
The environment and CRISP robot are visualized in Fig.~\ref{fig:crisp-all}.
As shown in Fig.~\ref{fig:crisp-f}, the total runtime in the CRISP experiment is significantly higher than in the point robot and manipulator scenarios. This is primarily due to the high computational cost of evaluating the CRISP robot's shape, making each node expansion much slower. As a representative instance, we consider a task in which the end-effector must move between two distant locations inside the pleural space. In this case, increasing the approximation factor reduces both the number of solutions and the overall runtime. Relative to \algname{Baseline}, we obtain a speedup of roughly~$5\times$.

\paragraph{Route planning with road types}
We used the OpenStreetMap road network of the Boulder Foothills, Colorado, and partitioned the edges into paved and unpaved. From this graph, we sampled 20 source--target pairs whose Euclidean distance is at least $25\,\mathrm{km}$, yielding challenging instances with comparable POFs (45$+$ solutions). Here, the number of hidden objectives is $m=3$. 
As indicated in Sec.~\ref{subsec:route},~$g_{\mathrm{con}}$ is non-monotone and is used by \algname{Baseline} as a solution objective. Thus, the lexicographic order used in \textsc{Open} matters, as extending nodes by a non-monotonically increasing order may result in the addition of non-optimal solutions to \textsc{Sols}, which implies longer runtimes.
Let $\mathrm{C},\mathrm{M},\mathrm{L}$ stand for hidden objectives $g_{\mathrm{con}},g_{\max},g_\ell$, respectively.
We evaluate three orders: $\mathrm{CLM}$, $\mathrm{MLC}$, and $\mathrm{MCL}$. As shown in Fig.~\ref{fig:runtime-vs-epslength}, \algname{ObjAgg} achieves average speedups of $2\times$, $10\times$ and~$10\times$ vs. hidden objective orderings  MLC, CLM and MCL, respectively.
This provides an insight: \algname{Baseline} is highly sensitive to the lexicographic order used when non-monotone objectives exist, whereas \algname{ObjAgg} is not.

\begin{figure}[t]
  \centering
  \begin{subfigure}[h]{.55\columnwidth}
    \centering
    \includegraphics[width=\linewidth]{\figType{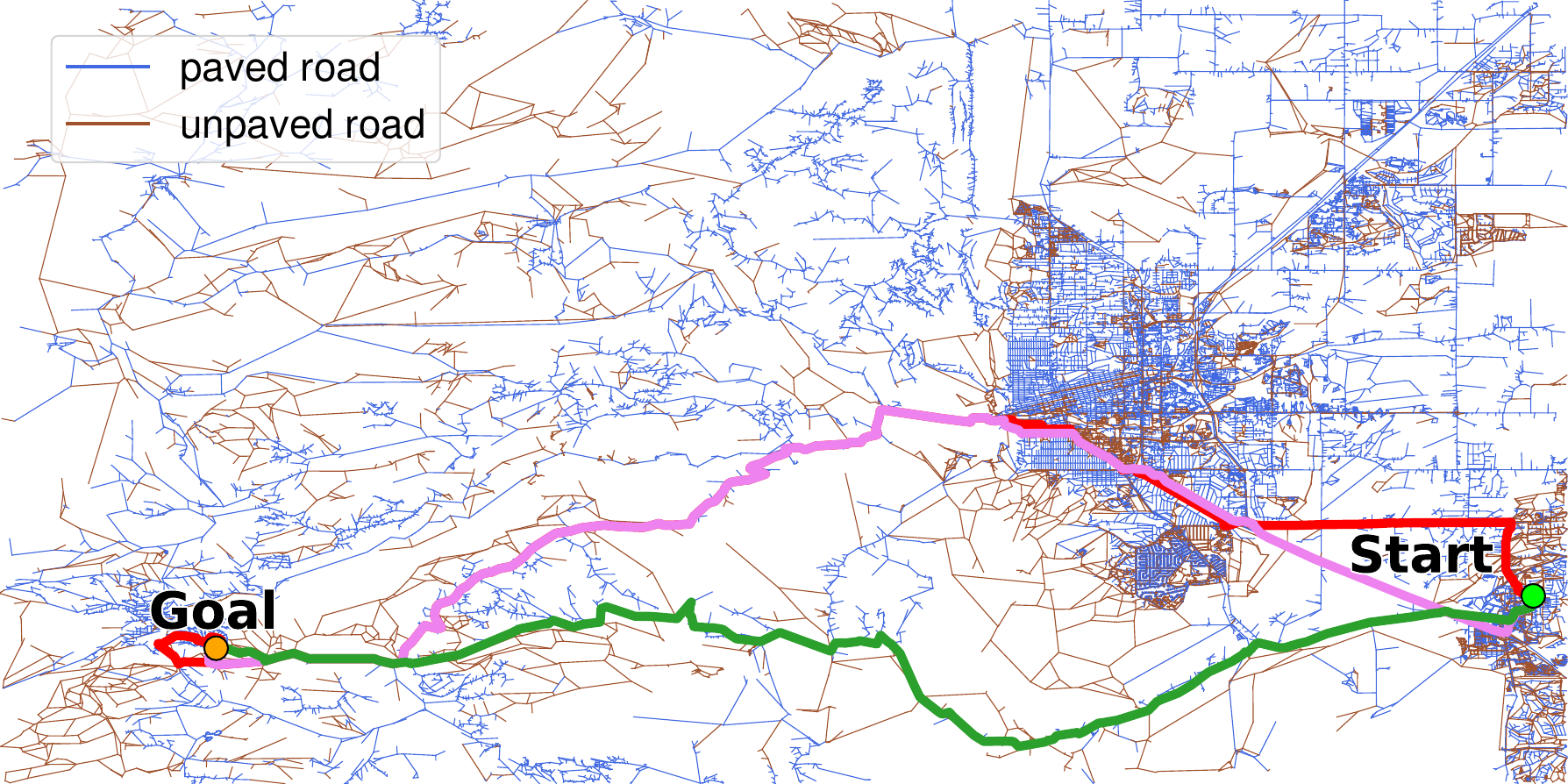}{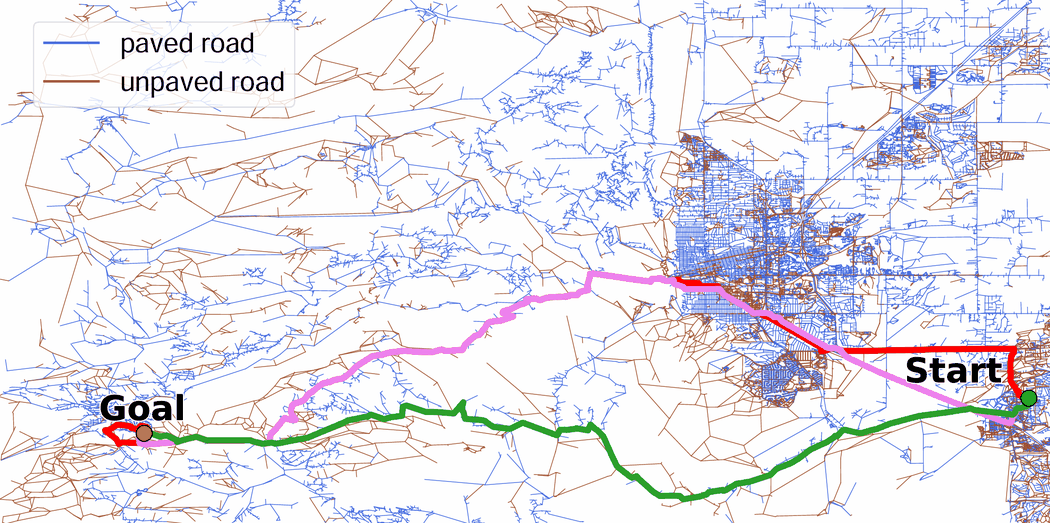}}
    \caption{}\label{fig:boulder-routes}
  \end{subfigure}\hfill
  \begin{subfigure}[h]{.43\columnwidth}
    \centering
    \includegraphics[width=\linewidth]{\figType{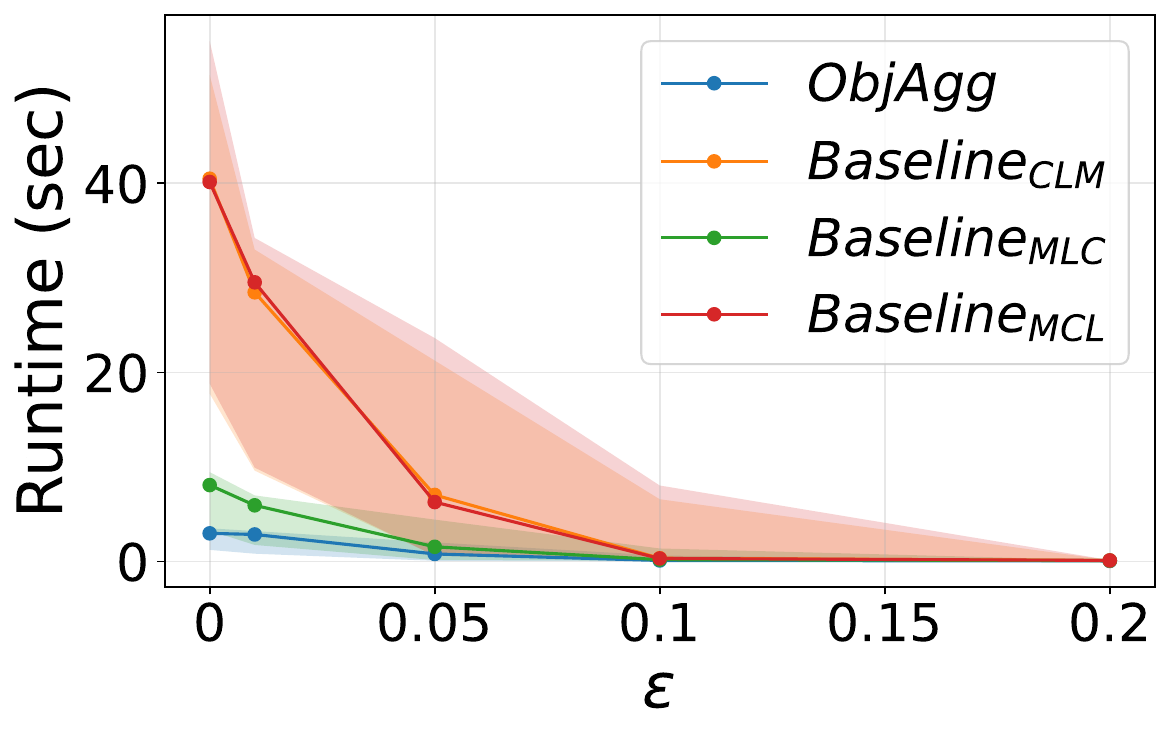}{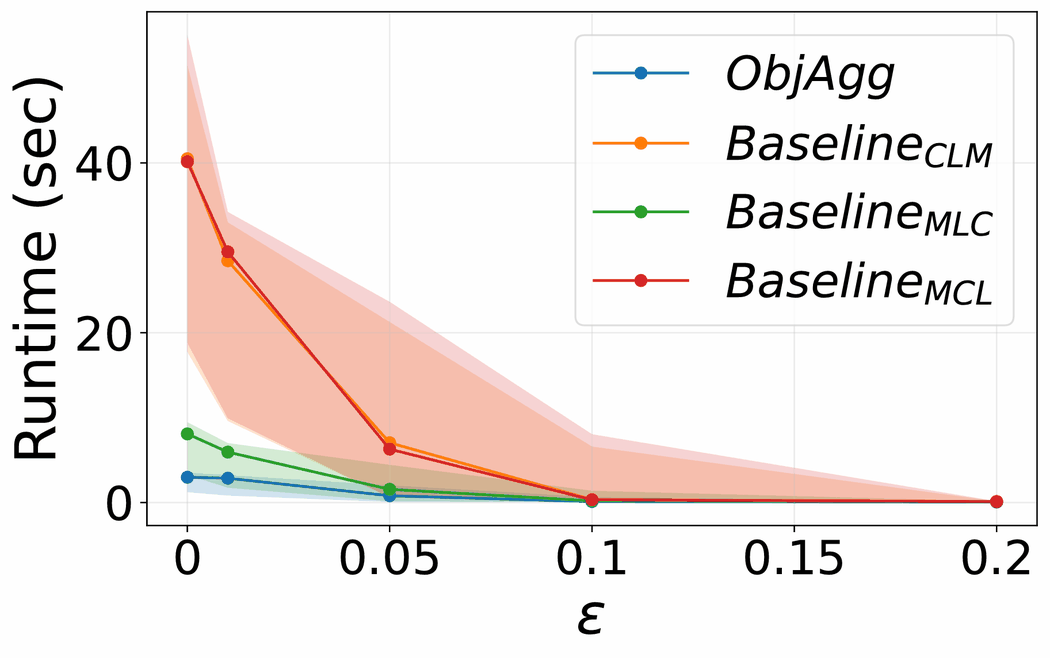}}
    \caption{}\label{fig:runtime-vs-epslength}
  \end{subfigure}
  \caption{
  (a) Road network with road types: paved (blue) and unpaved (brown) with three different paths (green, red, and pink) forming an approximate POF from start to goal.
  (b) Median runtime as a function of $\varepsilon$, shaded regions indicate the inter-quartile range (25th--75th percentiles). 
  }
  \label{fig:pareto_crisp}
  \vspace{-5mm}
\end{figure}



\section{Conclusion and Future Work}
\label{sec:conclusions}
The objective aggregation framework introduced in this paper  offers a unifying abstraction that is both theoretically clean and practically impactful. It allows to naturally model common real-world robotic problems that cannot be efficiently solved using existing MOS machinery.
However, despite appearing simple, there are interesting questions that arise:
Are there classes of aggregation functions that make the MOS problem provably easier or harder?
Are there principled methods for designing admissible and informative heuristics under non-trivial aggregation?
Can the framework be extended to richer problem classes?
We believe that this paper opens many exciting avenues for future work.

\bibliographystyle{plain}
\bibliography{ref}

 \newpage
 \section{Supplementary Material}

We restate Thm.~1 and provide a proof for it.
\begin{theorem*}
    Let \algname{ALG} be some MOS algorithm and consider a MOS problem with $k$ and $m$ solution and hidden objectives for some~$k<m$, respectively, and with 
    monotonically non-decreasing path extension and objective-aggregation functions $\mathcal{F}_{\rm ext}$ and $\mathcal{F}_{\rm agg}$.
    If $(m,m)$-\algname{ALG} returns a POF to the MOS problem
    then $(k,m)$-\algname{ALG} (using the {\color{magenta} magenta} text instead of the {\color{blue} blue} text in Alg.~\ref{alg:mos}) also returns a POF to the MOS problem.
\end{theorem*}

\begin{proof}
    First, note that the lexicographic key used to sort $\textsc{Open}$ has no effect on the completeness of the algorithm, as the search terminates only once $\textsc{Open}$ turns empty, and paths can only be discarded if they are dominated by a different solution or a different path. Hence, regardless of the order of nodes in $\textsc{Open}$ the algorithm $(k,m)$-\algname{ALG} will not terminate before exhausting all possible un-dominated paths.

    Second, the \texttt{is\_dominated\_path} is the same as in $(m,m)$-\algname{ALG}, so paths that are not discarded by $(m,m)$-\algname{ALG} are also not discarded by $(k,m)$-\algname{ALG}.

    Third, since $\mathcal{F}_{\rm ext}$ and $\mathcal{F}_{\rm agg}$ are monotonically non-decreasing their composition $\cost_{\mathcal{F}_{\rm agg}}():=\mathcal{F}_{\rm agg}(\cost_{\mathcal{F}_{\rm ext}}())$ is also monotonically non-decreasing.
    Thus, $\mathbf{g}$ is also monotonically non-decreasing, and given that $\mathbf{h}$ is admissible (as we assume throughout the paper), $\mathbf{f}$ is also monotonically non-decreasing.
    Hence, it holds that paths are encountered in a monotonically non-decreasing order. Namely, solutions that belong to the POF are encountered before non-optimal solutions that do not belong to the POF. Combining this with the fact that $(k,m)$-\algname{ALG} will not terminate before exhausting all possible un-dominated paths (proven above), we can determine that all optimal paths will be found and inserted into $\textsc{SOLS}$.

    Fourth, as \texttt{is\_dominated\_sol} utilizes $\cost_{\mathcal{F}_{\rm agg}}$ for its domination test, every non-optimal solution (where optimality is w.r.t. the solution objectives) will be discarded if an optimal solution was already found. Combining this with the fact that optimal solutions are guaranteed to be found before their corresponding non-optimal solutions (proven above), we can determine that $(k,m)$-\algname{ALG} returns a POF to the MOS problem defined by the solution objectives.
    Concluding, if $(m,m)$-\algname{ALG} returns a POF to a MOS problem
    then $(k,m)$-\algname{ALG} (using the {\color{magenta} magenta} text instead of the {\color{blue} blue} text in Alg.~\ref{alg:mos}) also returns a POF to its MOS problem.
\end{proof}

\end{document}